\documentclass{article}

\usepackage{microtype}
\usepackage{graphicx}
\usepackage{subfigure}
\usepackage{booktabs} % for professional tables

\usepackage{hyperref}
\usepackage{tcolorbox}
\usepackage{bm}

\usepackage[accepted]{icml2023}

\usepackage{amsmath}
\usepackage{amssymb}
\usepackage{mathtools}
\usepackage{amsthm}

\usepackage[capitalize,noabbrev]{cleveref}

\theoremstyle{plain}
\newtheorem{theorem}{Theorem}[section]

\newtheorem{lemma}[theorem]{Lemma}

\theoremstyle{definition}
\newtheorem{definition}[theorem]{Definition}

\theoremstyle{remark}
\newtheorem{remark}[theorem]{Remark}

\usepackage[textsize=tiny]{todonotes}

\newcommand{\lip}{\mathrm{Lip}}
\newcommand{\poly}{\mathrm{poly}}
\newcommand{\diam}{\mathrm{diam}}
\newcommand{\var}{\mathrm{Var}}

\newcommand{\Appendix}[1]{the full version for}

\renewcommand{\a}{\mathbf{a}}

\newcommand{\R}{\mathbb{R}}

\renewcommand{\comment}[1]{}

\newcommand{\cA}{\mathcal{A}}

\newcommand{\cC}{\mathcal{C}}
\newcommand{\cD}{\mathcal{D}}
\newcommand{\cF}{\mathcal{F}}

\newcommand{\cL}{\mathcal{L}}

\newcommand{\cU}{\mathcal{U}}

\newcommand{\cN}{\mathcal{N}}
\newcommand{\cO}{\mathcal{O}}

\newcommand{\cX}{\mathcal{X}}
\newcommand{\cW}{\mathcal{W}}

\newcommand{\bbE}{\mathbb{E}}

\newcommand{\bbR}{\mathbb{R}}

\DeclareMathOperator*{\argmin}{argmin}

\icmltitlerunning{A Law of Robustness beyond Isoperimetry}

\begin{document}

\twocolumn[
\icmltitle{A Law of Robustness beyond Isoperimetry}

% It is OKAY to include author information, even for blind
% submissions: the style file will automatically remove it for you
% unless you've provided the [accepted] option to the icml2023
% package.

% List of affiliations: The first argument should be a (short)
% identifier you will use later to specify author affiliations
% Academic affiliations should list Department, University, City, Region, Country
% Industry affiliations should list Company, City, Region, Country

% You can specify symbols, otherwise they are numbered in order.
% Ideally, you should not use this facility. Affiliations will be numbered
% in order of appearance and this is the preferred way.
\icmlsetsymbol{equal}{*}

\begin{icmlauthorlist}
\icmlauthor{Yihan Wu}{yyy}
\icmlauthor{Heng Huang}{yyy}
\icmlauthor{Hongyang Zhang}{comp}
%\icmlauthor{}{sch}
%\icmlauthor{}{sch}
\end{icmlauthorlist}

\icmlaffiliation{yyy}{Department of Computer Science, University of Maryland at College Park}
\icmlaffiliation{comp}{School of Computer Science, University of Waterloo}

\icmlcorrespondingauthor{Yihan Wu}{ywu42@umd.edu}
\icmlcorrespondingauthor{Heng Huang}{heng@umd.edu}
\icmlcorrespondingauthor{Hongyang Zhang}{hongyang.zhang@uwaterloo.ca}

% You may provide any keywords that you
% find helpful for describing your paper; these are used to populate
% the "keywords" metadata in the PDF but will not be shown in the document
\icmlkeywords{Machine Learning, ICML}

\vskip 0.3in
]

% this must go after the closing bracket ] following \twocolumn[ ...

% This command actually creates the footnote in the first column
% listing the affiliations and the copyright notice.
% The command takes one argument, which is text to display at the start of the footnote.
% The \icmlEqualContribution command is standard text for equal contribution.
% Remove it (just {}) if you do not need this facility.

%\printAffiliationsAndNotice{}  % leave blank if no need to mention equal contribution
\printAffiliationsAndNotice{} % otherwise use the standard text.

\begin{abstract}
We study the \textit{robust interpolation problem} of arbitrary data distributions supported on a bounded space and propose a two-fold law of robustness. Robust interpolation refers to the problem of interpolating $n$ noisy training data points in $\R^d$ by a Lipschitz function. Although this problem has been well understood when the samples are drawn from an isoperimetry distribution, much remains unknown concerning its performance under generic or even the worst-case distributions. We prove a Lipschitzness lower bound $\Omega(\sqrt{n/p})$ of the interpolating neural network with $p$ parameters on arbitrary data distributions. With this result, we validate the law of robustness conjecture in prior work by Bubeck, Li, and Nagaraj on two-layer neural networks with polynomial weights. We then extend our result to arbitrary interpolating approximators and prove a Lipschitzness lower bound $\Omega(n^{1/d})$ for robust interpolation. Our results demonstrate a two-fold law of robustness: i) we show the potential benefit of overparametrization for smooth data interpolation when $n=\poly(d)$, and ii) we disprove the potential existence of an $\cO(1)$-Lipschitz robust interpolating function when $n=\exp(\omega(d))$. 

% 2) Small data hurt robustness: $n=\exp(\Omega(d))$ is necessary for obtaining a good population error under certain distributions by any $\cO(1)$-Lipschitz learning algorithm. 
% Perhaps surprisingly, our results shed light on the curse of big data and the blessing of dimensionality for robustness.
\end{abstract}
\vspace{-0.2cm}
\section{Introduction}
\vspace{-0.1cm}

Robustness has been a central research topic in machine learning~\cite{szegedy2013intriguing,goodfellow2014explaining}, statistics~\cite{huber2004robust}, operation research~\cite{ben2009robust}, and many other domains. In machine learning, study of adversarial robustness has led to significant advances in defending against adversarial attacks, where test inputs with slight modification can lead to problematic prediction results. In statistics and operation research, robustness is a desirable property for optimization problems against uncertainty, which can be represented as deterministic or random variability in the value of optimization parameters. This is known as robust statistics or robust optimization. In both cases, the problem can be stated as given a deterministic labeling function $g:\R^d\rightarrow [-1,1]$, (approximately) interpolating the training data $\{(x_i,g(x_i))\}_{i=1}^n$ or its noisy counterpart by a function with small Lipschitz constant. The focus of this paper is on the latter setting known as \emph{robust interpolation problem}~\cite{bubeck2021universal}. That is, given noisy training data $\{(x_i,g(x_i)+z_i)\}_{i=1}^n$ of size $n$ where $x_1,\cdots,x_n$ are restricted in a unit ball and $z_1,\cdots,z_n$ have variance $>0$, how many network parameters and training samples are needed for robust interpolation provided that the functions in the class can (approximately) interpolate the noisy training data with Lipschitz constant $L$?

There are several reasons to study the noisy setting~\cite{bubeck2021universal}: 1) The real-world data are  noisy. For example, it has been shown that around 3.3\% of the data in the most-cited datasets was inaccurate or mislabeled~\cite{northcutt2021pervasive}. 2) This noise assumption is necessary from a theoretical point of view, as otherwise there could exist a Lipschitz function which perfectly fits the training data for any large $n$. Despite progress on the robust interpolation problem~\cite{bubeck2021universal,pmlr-v134-bubeck21a}, many fundamental questions remain unresolved. In modern learning theory, it was commonly believed that 1) big data~\cite{schmidt2018adversarially}, 2) low dimensionality of input~\cite{blum2020random,yang2020randomized,kumar2020curse}, and 3) overparametrization~\cite{bubeck2021universal,pmlr-v134-bubeck21a} improve robustness. We view the robustness problem from the perspective of Lipschitzness and ask the following question:
\begin{center}
\emph{Are big data and large models a remedy for robustness?}
\end{center}
% \begin{center}
% \begin{tcolorbox}[colback=blue!5!white,colframe=blue!75!black]
% \vspace{-0.05in}
% \textbf{Are big data and large models a remedy for robustness?}
% \vspace{-0.05in}
% \end{tcolorbox}
% \end{center}
In fact, there is significant empirical evidence to indicate that enlarging the model size (overparametrization) improves robustness when $n$ is moderately large (\emph{e.g.}, when $n=\poly(d)$, see \cite{madry2017towards,schmidt2018adversarially}). Our work verifies the benefit of overparametrization for fitting a neural network with $p$ parameters below the noise level by proving such neural networks must have a Lipschitzness lower bound $\Omega(\sqrt{n/p})$. 
On the other hand, big data and large models may not be a remedy for robustness if $n$ goes even larger. We show that for any approximator, no matter how many parameters it contains, its Lipschitzness is of order $\Omega(n^{1/d})$. In particular, our result disproves the existence of learning an $\cO(1)$-Lipschitz function with $n=\exp(\omega(d))$. 
Besides, by showing that for any learning algorithm, there exists a joint data distribution such that one needs at least $n=\exp(\Omega(d))$ samples to learn an $\cO(1)$-Lipschitz function with good population error, we demonstrate that big data are also necessary for robust interpolation in some special cases.

The robust interpolation problem becomes more challenging when no assumptions are made on the distribution of covariates. Due to the well-separated nature of data, most positive results for obtaining good Lipschitzness lower bound have focused on the isoperimetry distribution~\cite{bubeck2021universal}. A probability measure $\mu$ on $\bbR^d$ satisfies $c$-isoperimetry if for any bounded $L$-Lipschitz $f:\bbR^d \to\bbR$, and any $t\geq0$, 
\vspace*{-0.2cm}
$$\Pr(|f(x)-\bbE[f(x)]|\geq t)\leq 2\exp(-\frac{dt^2}{2cL^2}).$$ 
Isoperimetry states that the output of any Lipschitz function is $\cO(1)$-subgaussian under suitable rescaling. Special cases of isoperimetry include high-dimensional Gaussians $\cN(0,\frac{I_d}{d})$, uniform distributions on spheres and hypercubes of diameter 1. However, real-world data might not follow the isoperimetry assumption. 
In contrast, our results of Theorem \ref{thm:result1} go beyond isoperimetry and providing a lower bound of robustness for functions with $p$ parameters under \emph{arbitrary} distributions in the bounded space. Our results of Theorem \ref{theorem: main result} go even further by providing a universal lower bound of robustness for \emph{any} model class, including the class of neural networks with arbitrary architecture.

\textbf{Notations.} We will use $\cX$ to represent the instance space, $\mathcal{F}=\{f: \mathcal{X}\to[-1,1]\}$ to represent the hypothesis/function space, $x\in\cX$ to represent the sample instance, $y\in[-1,1]$ to represent the target, and $z$ to represent the target noise. For errors, denote by $l(f(x),y)$ the loss function of $f$ on instance $x$ and target $y$, in our work we use the mean squared error as in \citet{bubeck2021universal}, i.e., $l(f(x),y)=(f(x)-y)^2$. Let $\cL_{\cD}(f) := \mathbb{E}_{(x,y)\sim {\cD}}[l(f(x),y)]$ be the population error, and let $\cL_S(f) := \frac{1}{|S|}\sum_{(x,y)\in S}[l(f(x),y)]$ be the empirical error. Denote by $f:\cX\rightarrow [-1,1]$ the \emph{prediction function} which maps an instance to its predicted target. It can be parameterized, e.g., by deep neural networks. For norms, we denote by $\|x\|$ a generic norm. Examples of norms include $\|x\|_\infty$, the infinity norm, and $\|x\|_2$, the $\ell_2$ norm. We will frequently use $(\cX,\|\cdot\|)$ to represent the normed linear space of $\cX$ with norm $\|\cdot\|$. Define $\diam(\cX)$ as the diameter of $\cX$ w.r.t. the norm $\|\cdot\|$. For a given score function $f$, we denote by $\lip_{\|\cdot\|}(f)$ (or sometimes $\lip(f)$ for simplicity) the Lipschitz constant of $f$ w.r.t. the norm $\|\cdot\|$. Let $\lceil\cdot\rceil$ represent the ceiling operator. We will use $\cO(\cdot)$, $\Theta(\cdot)$ $o(\cdot)$, and $\Omega(\cdot)$ to express sample complexity and Lipschitzness.
% and $\widetilde\cO(\cdot)$ and $\widetilde\Omega(\cdot)$ to ignore the $\ln(\cdot)$ factors.

\subsection{Our results}

Our law of robustness is two-fold: 
a) overparametrization can potentially help robust interpolation when $n=\poly(d)$ (Section \ref{sec:beyondiso}), and b) there exists no robust interpolation when $n=\exp(\omega(d))$ (Section \ref{section: robustness suffers from too many data}).

Lipschitzness (or local Lipschitzness) is an important characterization of adversarial robustness for learning algorithms~\cite{yang2020closer,zhang2019theoretically,wu2022adversarial,wu2022towards}. The popular randomized smoothing approaches \cite{cohen2019certified,li2019certified,wu2022retrievalguard} can provide robust guarantee through Lipschitzness but suffer curse of dimensionality problem \cite{wu2021completing}. Thus, studying the Lipschitzness is crucial for understanding robustness. For a given score function $f$, we denote by $\lip_{\|\cdot\|}(f)$ the Lipschitz constant of $f$ \emph{w.r.t.} the norm $\|\cdot\|$. That is, for any $x_1,x_2$ in the input space, $|f(x_1)-f(x_2)|\le \lip_{\|\cdot\|}(f)\|x_1-x_2\|$.
Our results show lower bounds on the Lipschitzness of learned functions when the training error is slightly smaller than the noise level (\emph{i.e.}, in the case of overfitting), but without assumptions on the distribution of covariates except that they are restricted in the bounded space $\cX:=\{x:\|x\|\le 1\}$. We are interested in the assumption of bounded space because: 1) most applications of machine learning focus on the case where the data are in the bounded space. For example, images and videos are considered to be in $[-1,1]^d$. 2) The discussion of Lipschitzness is closely related to how large the input space is. For example, for the images restricted in $[-1,1]^d$, special attentions are paid on the $\ell_\infty$ robust radius of 0.031 or 0.062~\cite{zhang2019theoretically,madry2017towards}, which corresponds to a (local) Lipschitz constant of $\cO(1)$ for the classifier.

\noindent{\textbf{Overparametrization may benefit robust interpolation.}}
The universal law of robustness by \citet{bubeck2021universal} provides an $\Omega(\sqrt{nd/p})$ Lipschitzness lower bound of the interpolating functions when the underlying distribution is isoperimetry (see Theorem \ref{theorem: bubeck's result}). Our first result goes beyond the isoperimetry assumption, and provides an $\Omega(\sqrt{n/p})$ Lipschitzness lower bound of the interpolating functions under arbitrary distribution. We note that the $\sqrt{d}$ difference between the two Lipschitzness lower bounds is due to the special property of the isoperimetry assumption (see Remark~\ref{rmk:difference}). Our result predicts the \emph{potential} existence of an $\cO(1)$-Lipschitz function that fits the data below the noise level when $p=\Omega(n)$.
The following informal theorem illustrates the results (the detailed theorems are introduced at later sections):

\noindent{\textbf{Theorem A (informal version of Theorem \ref{thm:result1}).} }
\emph{Let $\cF$ be any class of functions from $\R^d\rightarrow [-1,1]$ and let $\{(x_i,y_i)\}_{i=1}^n$ be i.i.d. input-output pairs in $\{x:\|x\|\le 1\}\times [-1,1]$ for any given norm $\|\cdot\|$. Assume that:
\begin{itemize}
\vspace{-0.25cm}
\item[1.]
The expected conditional variance of the output (i.e., the ``noise level'') is strictly positive, denoted by $\sigma^2:=\bbE[\var[y|x]]>0$.
\vspace{-0.25cm}
\item[2.]
$\cF$ admits a $J$-Lipschitz parametrization by $p$ real parameters, each of size at most $poly(n,d)$.
\end{itemize}
\vspace{-0.25cm}
Then, with high probability over the sampling of the data, one has simultaneously for all $f\in\cF$:
\begin{equation*}
\frac{1}{n}\hspace{-0.1cm}\sum_{i=1}^n(y_i-f(x_i))^2\hspace{-0.1cm}\le\hspace{-0.1cm} \sigma^2-\epsilon \Rightarrow \lip_{\|\cdot\|}(f)\hspace{-0.1cm}\ge\hspace{-0.1cm} \Omega\hspace{-0.1cm}\left(\hspace{-0.1cm}\epsilon \sqrt{\frac{n}{p}}\right).
\end{equation*}}
\begin{remark}
Our theorem takes a further step in proving the Conjecture 1 in \citet{pmlr-v134-bubeck21a}, where it is conjectured that for generic data sets, with high probability, any $f$ in the collections of two layer networks with $p$ parameters fitting the data must also satisfy $\lip_{\|\cdot\|}(f)\ge\Omega(\sqrt{n/p})$. We validate the conjecture under the polynomial weights assumption, where
\citet{bubeck2021universal} validate the Conjecture 1 under the polynomial weights assumption and the isoperimetry assumption.
\end{remark}

\begin{remark}[Strong overparametrization is not necessary for the robust interpolation]
The Lipschitzness lower bound of \citet{bubeck2021universal} suggests strong overparametrization, i.e., $p=\Omega(nd)$, is required for the robust interpolation under the isoperimetry assumption. Our theorem shows that strong overparametrization may not be a necessary condition for the robust interpolation on a general distribution. Moderate overparametrization with $p=\Omega(n)$ may also be enough for robust interpolation. Our results are consistent with the empirical observations that CIFAR10 ($50000$ images) can be robustly fitted by a model with $p=10^6$, and ImageNet ($10^7$ images) can be robustly fitted by a model with $p=10^7\sim10^8$.
\end{remark}

\noindent{\textbf{Big data hurts robust interpolation.}}
Under the assumptions of isoperimetry distribution and the $J$-Lipschitz parameterized functions, the universal law of robustness by \citet{bubeck2021universal} predicts the \emph{potential} existence of an $\cO(1)$-Lipschitz function fits the data below the noise level when $p=\Omega(nd)$. Our result goes beyond the two assumptions and disproves the existence of such $\cO(1)$-Lipschitz functions in the big data scenario when $n=\exp(\omega(d))$ for \emph{arbitrary} distributions:

\noindent{\textbf{Theorem B (informal version of Theorem~\ref{theorem: main result}).}}
\emph{Let $\cF$ be any class of functions from $\R^d\rightarrow [-1,1]$ and let $\{(x_i,y_i)\}_{i=1}^n$ be i.i.d. input-output pairs in $\{x:\|x\|\le 1\}\times [-1,1]$ for any given norm $\|\cdot\|$. Assume that:
\begin{itemize}
\vspace{-0.25cm}
\item[1.]
The expected conditional variance of the output (i.e., the ``noise level'') is strictly positive, denoted by $\sigma^2:=\bbE[\var[y|x]]>0$.
\vspace{-0.25cm}
\end{itemize}
Then, with high probability over the sampling of the data, one has simultaneously for all $f\in\cF$:
\begin{equation*}
\frac{1}{n}\sum_{i=1}^n(y_i-f(x_i))^2\le \sigma^2-\epsilon\ \Rightarrow\  \lip_{\|\cdot\|}(f)\ge \Omega(\epsilon n^{1/d}).
\end{equation*}}

% \begin{remark}
% Our result is not contradict to the universal law of robustness \cite{bubeck2021universal} or our Theorem~\ref{thm:result1}, as they only 
% \end{remark}
\noindent\textbf{Difference between our results and~\citet{bubeck2021universal}.} \citet{bubeck2021universal} proposed a universal law of robustness for general class of functions (see Theorem \ref{theorem: bubeck's result}). Our results Theorem \ref{thm:result1} and Theorem \ref{theorem: main result} share the same setting with Theorem \ref{theorem: bubeck's result}, while the former ones make much weaker assumptions: 1) Both Theorem \ref{thm:result1} and Theorem \ref{theorem: main result} do not require an isoperimetry assumption of input distributions.
2) Theorem \ref{theorem: main result} does not make any assumption on the Lipschitzness and size of model parametrization.  Moreover, while Theorem \ref{theorem: bubeck's result} predicts potential existence of an $\cO(1)$-Lipschitz robust interpolating function when $p=\Omega(nd)$, Theorem \ref{theorem: main result} disproves the hypothesis in the big data scenario when $n=\exp(\omega(d))$ for \emph{arbitrary} distributions in the bounded space. Besides, our bounds work for all $\ell_p (p\geq 1)$ norm while the bound in \citet{bubeck2021universal} only focuses on $\ell_2$ norm.

\textbf{Practical implications.} Our analysis provides important implications for practical settings. When selecting the models for learning on a certain dataset, ideally the number of parameters in the selected model should be the same (or slightly larger) scale of the dataset in order to get good robust performance. When the size of dataset is too large comparing to the dimension of dataset, in order to achieve good robustness, it may be beneficial to either reduce the size of the training data or scatter the data in a higher-dimensional space by padding special covariates. This approach can help to mitigate the negative effects of the curse of big data and improve model robustness, particularly when dealing with large datasets in practical applications \cite{wu2022faster,wu2023decentralized}.

\vspace{-0.2cm}
\section{Related Work}

\vspace{-0.1cm}

\noindent{\textbf{Robust interpolation problem.}} \citet{pmlr-v134-bubeck21a} provided the first guarantee on the law of robustness for two-layer neural networks which was later extended by \citet{bubeck2021universal} to a universal law of robustness for general class of functions under isoperimetry distributions. A probability measure $\mu$ on $\bbR^d$ satisfies $c$-isoperimetry if for any bounded $L$-Lipschitz $f:\bbR^d \to\bbR$, and any $t\geq0$, 
$\Pr(|f(x)-\bbE[f(x)]|\geq t)\leq 2\exp(-\frac{dt^2}{2cL^2}).$ 

\begin{theorem}[Theorem 1 of \citet{bubeck2021universal}]
\label{theorem: bubeck's result}
Let $\cF$ be a class of functions from $\R^d\rightarrow [-1,1]$ and let $\{(x_i,y_i)\}_{i=1}^n$ be i.i.d. input-output pairs in $\R^d\times [-1,1]$. Assume that:
\begin{itemize}
\vspace{-0.25cm}
\item[1.]
The expected conditional variance of the output (i.e., the ``noise level'') is strictly positive, denoted by $\sigma^2:=\bbE[\var[y|x]]>0$.
\vspace{-0.25cm}
\item[2.]
$\cF$ admits a $J$-Lipschitz parametrization by $p$ real parameters, each of size at most $poly(n,d)$.
\vspace{-0.25cm}
\item[3.]
The distribution $\mu$ of the input $x_i$ satisfies isoperimetry (or a mixture thereof).
\end{itemize}
\vspace{-0.25cm}
Then, with high probability over the sampling of the data, one has simultaneously for all $f\in\cF$:
\begin{equation*}
\frac{1}{n}\hspace{-0.1cm}\sum_{i=1}^n(y_i-f(x_i))^2\hspace{-0.1cm}\le\hspace{-0.1cm} \sigma^2-\epsilon \Rightarrow \lip_{\|\cdot\|_2}(f)\hspace{-0.1cm}\ge\hspace{-0.1cm} \Omega\hspace{-0.1cm}\left(\hspace{-0.1cm}\epsilon \sqrt{\frac{nd}{p}}\right).
\end{equation*}
\end{theorem}
Our work extends the result of \citet{bubeck2021universal} by consequently removing the third assumption (see \autoref{thm:result1}) and the second assumption (see \autoref{theorem: main result}).

\noindent{\textbf{Sample complexity of robust learning.}} The sample complexity of robust learning for benign distributions and certain function class has been extensively studied in the recent years. In particular, \citet{bhattacharjee2021sample} considered the sample complexity of robust linear classification on the separated data. \citet{yin2019rademacher} studied the adversarially robust generalization problem through the lens of Rademacher complexity. \citet{cullina2018pac} extended the PAC-learning framework to account for the presence of an adversary. \citet{montasser2019vc} showed that any hypothesis class with finite VC dimension is robustly PAC learnable with an improper learning rule. They also showed that the requirement of being improper is necessary. \citet{schmidt2018adversarially} showed an $\Omega(\sqrt{d})$-factor gap between the standard and robust sample complexity for a mixture of Gaussian distributions in $\ell_\infty$ robustness, which was later extended to the case of $\ell_p$ robustness with a tight bound by \citet{bhagoji2019lower,dobriban2020provable,dan2020sharp}. Different from the prior work, our work is the first to discover the sample complexity of robust learning for \emph{arbitrary} function class and learning algorithms.

%%%%%%%%%%%%%%%%%%%%%%%%%%%%%%%%%%%%%%%%%%%%%%%%%%%%%
%%%%%%%%%%%%%%%%%%   MAIN RESULTS   %%%%%%%%%%%%%%%%%
%%%%%%%%%%%%%%%%%%%%%%%%%%%%%%%%%%%%%%%%%%%%%%%%%%%%%%

\vspace{-0.2cm}
\section{A Two-fold Law of Robustness}\label{sec:lawofrobust}
In this section, we present our main theoretical analysis, which contributes to our two-fold law of robustness. All missing proofs can be found in the appendix.

\textbf{Robust interpolation problem.} We first introduce our problem settings. Given noisy training data $\{(x_i,y_i:=g(x_i)+z_i)\}_{i=1}^n$ of size $n$ where $x_1,\dots,x_n$ are training samples, $g(x_1),\dots,g(x_n)$ the ground truth, and $z_1,\dots,z_n$ have variance $\sigma^2>0$, we say a model $f$ robustly interpolates (or fits the data below the noise level) the training data if and only if $$\exists \epsilon>0,\frac{1}{n}\sum_{i=1}^n(y_i-f(x_i))^2\le \sigma^2-\epsilon.$$

\textbf{Our two-fold law of robustness.}
a) Overparametrization can potentially help robust interpolation when $n=\poly(d)$ (Section \ref{sec:beyondiso}); b) There exists no robust interpolation when $n=\exp(\omega(d))$ (Section \ref{section: robustness suffers from too many data}).

\vspace{-0.2cm}
\subsection{A Lipschitz lower bound beyond the isoperimetry assumption.}\label{sec:beyondiso}
In this part, we show the first part of our two-fold law of robustness: overparametrization can potentially help robust interpolation when $n=\poly(d)$. Notice, here we claim ``potentially help'' as overparametrization is only a necessary but not sufficient condition for robust interpolation. 
% In Section~\ref{section: robustness suffers from too few data} we introduce a case, in which even overparametrization cannot help to obtain a robust interpolating function.

\textbf{Motivation.} We notice that, the proof of Theorem~\ref{theorem: bubeck's result} \cite{bubeck2021universal} depends heavily on the definition of isoperimetry distribution, i.e., $\Pr(|f(x)-\bbE[f(x)]|\geq t)\leq 2\exp(-\frac{dt^2}{2cL^2})$ for $L$-Lipschitz $f:\bbR^d\to\bbR$. This formula indicates the high-concentration property of isoperimetry distributions due to the $\exp(-d)$ dependency of $\Pr(|f(x)-\bbE[f(x)]|\geq t)$. The $\exp(-d)$ dependency is also the reason that the Lipschitzness lower bound of \citet{bubeck2021universal} is $\Omega(\sqrt{nd/p})$ instead of the $\Omega(\sqrt{n/p})$ lower bound we derived.

\textbf{Challenge.} One may naturally come up with the idea to derive a bound of $\Pr(|f(x)-\bbE[f(x)]|\geq t)$ for arbitrary distributions and go beyond the isoperimetry distribution. However, the challenge is that unlike the regular concentration bound on $\Pr(|x-\bbE[x]|\geq t)$, we are dealing with a more complicate case, where the random variable is $f(x)$ with arbitrary $L$-Lipschitz $f$. To solve this problem, we apply the Azuma's inequality below:

\begin{lemma}[Azuma's inequality \cite{azuma1967weighted}]
\label{lm:azuma} Suppose $\{X_{k}:k=0,1,2,3,\dots \}$ is a martingale and $|X_{k}-X_{k-1}|\leq c_{k}$
almost surely. Then for all positive integers $N$ and $\epsilon>0$,
    $$\Pr(|X_{N}-X_{0}|\geq \epsilon)\leq 2\exp \left(-\frac{\epsilon^{2}}{2\sum_{k=1}^{N}c_{k}^{2}}\right).$$
\end{lemma}
Azuma's inequality shows the concentration bound for the values of martingales that have bounded differences. With this lemma, we are able to derive the following concentration bound for arbitrary distributions on a bounded space.
\begin{lemma}\label{lm:concentration}
Given an arbitrary probability measure $\mu$ on the bounded space $\cX\subset\bbR^d$, for any $L$-Lipschitz $f:\bbR^d \to\bbR$, and any $t\geq0$, 
$$\Pr(|f(x)-\bbE[f(x)]|\geq t)\leq 2\exp(-\frac{t^2}{2\diam(\cX)^2L^2}).$$
\end{lemma}
% \begin{proof}
%    missing proof 
% \end{proof}
Comparing with the $\exp(-\frac{dt^2}{2cL^2})$ bound for the isoperimetry distributions, our bound for arbitrary distributions only differs a $d$ on the numerator of the term inside the exponential. In order to achieve the same concentration bound of isoperimetry distributions, one need $\diam(\cX)=\Theta(1/\sqrt{d})$, which means our input are located on an $\Theta(1/\sqrt{d})$-diameter space. As the real world datasets are usually supported on an $\Theta(1)$-diameter space, matching the isoperimetry bound for all distributions is empirical meaningless. 

With Lemma~\ref{lm:concentration}, we can start to calculate the Lipschitzness lower bound with the following lemma on finite function class
\begin{lemma}\label{lm:finitefunction}
    Let $\cF$ be a finite class of $L$-Lipschitz functions from $\R^d\rightarrow [-1,1]$ and let $\{(x_i,y_i)\}_{i=1}^n$ be i.i.d. input-output pairs in $\{x:\|x\|\le 1\}\times [-1,1]$ for any given norm $\|\cdot\|$. Assume that the expected conditional variance of the output (i.e., the ``noise level'') is strictly positive, denoted by $\sigma^2:=\bbE[\var[y|x]]>0$, we have 
    \begin{equation*}
    \begin{split}
&\Pr\left(\exists f\in\cF:\frac{1}{n}\sum_{i=1}^n (y_i-f(x_i))^2\le\sigma^2-\epsilon\right)\\ 
\le &4\exp\left(-\frac{n\epsilon^2}{8^3}\right)+|\cF|\exp\left(-\frac{\epsilon^2n}{2^{10}L^2}\right).
\end{split}
\end{equation*}
\end{lemma}

Lemma~\ref{lm:finitefunction} shows the connection between the robust interpolation problem and the Lipschitzness of the underlying functions. Notice, the probability of $\exists f\in\cF:\frac{1}{n}\sum_{i=1}^n (y_i-f(x_i))^2\le\sigma^2-\epsilon$ decreases with $L$, which indicates that we need a large enough $L$ to make sure that there exists $f$ satisfying the condition of robust interpolation problem. 
With this intuition, we can calculate the following Lipschitzness lower bound for the robust interpolation problem without the isoperimetry assumption.
\begin{theorem}
\label{thm:result1}
Let $\cF$ be any class of functions from $\R^d\rightarrow [-1,1]$ and let $\{(x_i,y_i)\}_{i=1}^n$ be i.i.d. input-output pairs in $\{x:\|x\|\le 1\}\times [-1,1]$ for any given norm $\|\cdot\|$. Assume that:
\begin{itemize}
\vspace{-0.25cm}
\item[1.]
The expected conditional variance of the output (i.e., the ``noise level'') is strictly positive, denoted by $\sigma^2:=\bbE[\var[y|x]]>0$.
% \vspace{-0.25cm}
\item[2.]
$J$-Lipschitz parametrization: $\cF = \{f_w,w \in \cW\}$ with $\cW \subset \bbR^p$, $\diam(\cW) \leq W$ and for any $w_1,w_2 \in W$,
$$||f_{w_1} - f_{w_2} ||_\cF \leq J||w_1 - w_2||.$$
\end{itemize}
\vspace{-0.25cm}
Then, with probability at least $1-\delta$, one has simultaneously for all $f\in\cF$:
\begin{equation*}
\begin{split}
&\frac{1}{n}\sum_{i=1}^n(y_i-f(x_i))^2\le \sigma^2-\epsilon \Rightarrow\\  &\lip_{\|\cdot\|}(f)\ge \frac{\epsilon}{32}\sqrt{\frac{n}{p\ln(36WJ\epsilon^{-1})+\ln(2/\delta)}}.
\end{split}
\end{equation*}
\end{theorem}
% \begin{proof}
%     missing proof
% \end{proof}
The crucial part of the proof is to find a finite $\epsilon/6J$-covering of $\cF$ with the $J$-Lipschitz parametrization assumption. Then we can apply Lemma~\ref{lm:finitefunction} to this finite covering set and get the Lipschitzness lower bound. We will show in the next section that without the $J$-Lipschitz parametrization assumption, one can hardly use the similar proof technique to derive the Lipschitzness lower bound.

\citet{bubeck2021universal} showed that under neural network settings, $J$ is always of polynomial order of the diameter of the weight space. Thus, if the weight is only polynomial large w.r.t. $d$ and $n$, $\ln(60WJ\epsilon^{-1})$ would not affect the Lipschitzness bound too much and we may neglect it in its asymptotic approximation. Thus, we have a Lipschitzness lower bound of order $\Omega(\epsilon\sqrt{n/p})$ for the robust interpolation problem. Our theorem validates the first part of our law of robustness, i.e., the potential existence of robust interpolating functions under the overparametrization scenario when $n=\poly(d)$ (see Remark~\ref{rmk:lawpart1}).

\textbf{Tightness of our bound}. When $n=\poly(d)$, Theorem 4 of \citet{pmlr-v134-bubeck21a} has already demonstrated the existence of an at most $\cO(\sqrt{n/p})$-Lipschitz two layer network, which fits generic data below the noise level. Thus, our Lipschitzness lower bound is tight.

\begin{remark}[Difference of the $\sqrt{d}$-dependency between Theorem~\ref{theorem: bubeck's result} and \ref{thm:result1}]
\label{rmk:difference} Comparing to the $\Omega(\epsilon\sqrt{nd/p})$ of Lipschitzness lower bound in Theorem~\ref{theorem: bubeck's result}, our bound does not depend on the dimension $d$. This difference, as we discussed in Lemma~\ref{lm:concentration}, is due to the isoperimetry assumption. In \citet{pmlr-v134-bubeck21a}, it's also showed that the tight Lipschitzness lower bound of two layer networks is of order $\Omega(\epsilon\sqrt{n/p})$, which is consistent with our results.
\end{remark}

% \begin{remark}[Difference of the $\sqrt{d}$-dependency between Theorem~\ref{theorem: bubeck's result} and \ref{thm:result1}]
% \label{rmk:difference} Comparing to the $\Omega(\epsilon\sqrt{nd/p})$ of Lipschitzness lower bound in Theorem~\ref{theorem: bubeck's result}, our bound does not depend on the dimension $d$. This difference, as we discussed in Lemma~\ref{lm:concentration}, is due to the isoperimetry assumption. In \citet{pmlr-v134-bubeck21a}, it's also showed that the tight Lipschitzness lower bound of two layer networks is of order $\Omega(\sqrt{n/p})$, which is consistent with our results.
% \end{remark}
% \vspace{-0.2cm}
\subsection{A Lipschitz lower bound beyond the $J$-Lipschitz parametrization assumption.}
\label{section: robustness suffers from too many data}

In this part, we show the second part of our two-fold law of robustness. We demonstrate an intriguing observation that huge data hurts robust interpolation. Our analysis leads to a universal lower bound of Lipschitzness regarding the robust interpolation problem, which goes beyond the isoperimetry and $J$-Lipschitz parametrization assumptions. Our analysis is based on the relation between Rademacher complexity and the generalization gap between the population error $\cL_{\cD}(f)$ and the training error $\cL_{S}(f)$. %We defer the complete proof to \autoref{sec:proof3.1}.

\textbf{Motivation.} The $J$-Lipschitzness parametrization assumption provides us a simple way to find a covering of the function space $\cF$. Although the Lipschitzness lower bound in \citet{bubeck2021universal} has only logarithmic dependency with respect to $J$, it may still affect the Lipschitzness lower bound when the weight of neural networks is exponentially large w.r.t. $d$, or the number of layers of neural works is polynomial w.r.t. $d$.  Thus, we seek to derive a Lipschitzness lower bound beyond the $J$-Lipschitzness parametrization assumption.

\textbf{Challenge.} Without the $J$-Lipschitzness parametrization assumption, the covering number of the function $\cF$ will have more complicate dependency on the Lipschitzness $L$ (see Lemma~\ref{lm: Covering number of functional space}). In this case, calculating Lipschitzness lower bound with Lemma~\ref{lm: Covering number of functional space} and Lemma~\ref{lm:finitefunction} requires one to solve an inequality like $L^{-d}\ln L+L^{-2}\geq C$, which obviously has no closed-form solution when $d\geq3$. Thus, we need other techniques to deal with this case. Recall the objective of robust interpolation problem is 
$\frac{1}{n}\sum_{i=1}^n(y_i-f(x_i))^2\le \sigma^2-\epsilon,$
one can immediately find that the left hand side formula is the train error with mean squared loss $\cL_S(f)$. Under the label noise settings, we have $\cL_{\cD}(f)=\bbE_{\cD}[(f(x)-y)^2] \geq \bbE_{x}[\var(y|x)]=\sigma^2,$ which yields $\frac{1}{n}\sum_{i=1}^n(y_i-f(x_i))^2\le \sigma^2-\epsilon\Rightarrow \cL_S(f)\leq\cL_{\cD}(f)-\epsilon.$ Therefore, if one can derive 
$$\cL_S(f)\leq\cL_{\cD}(f)-\epsilon\Rightarrow \lip_{\|\cdot\|}(f)\ge\Omega(\epsilon n^{1/d}),$$
a natural corollary is that 
\begin{equation*}
    \begin{split}
        \frac{1}{n}\sum_{i=1}^n(y_i-f(x_i))^2\le \sigma^2-\epsilon &\Rightarrow\cL_S(f)\leq\cL_{\cD}(f)-\epsilon \\
        &\Rightarrow \lip_{\|\cdot\|}(f)\ge\Omega(\epsilon n^{1/d}).
    \end{split}
\end{equation*}
In this way, we successfully convert the robust interpolation problem to a generalization problem between the empirical error and population error under the mean squared loss, which can be solved by the statistical learning techniques, e.g., VC dimension and Rademacher complexity. We focus on the Rademacher complexity in this part.

\textbf{Rademacher complexity.} We start with the definition of Rademacher complexity, which measures the richness of a function class. For a set $\cA\subset\bbR^n$, the Rademacher complexity is defined as
$$R(\cA):= \frac{1}{n}\mathbb{E}_{{\sigma_1,...,\sigma_n}\in\{-1,1\}}\left[\sup_{\a\in \cA}\sum_{i=1}^n\sigma_ia_i\right].$$
Given a loss function $l$, a hypothesis class $\cF$, and a training set $S=\{(x_1,y_1),...,(x_n,y_n)\}$, denote by $l\circ\cF := \{l(f(\cdot),\cdot): f\in \cF\}$ and $l\circ\cF\circ S := \{(l(f(x_1),y_1),...,l(f(x_n),y_n)): f\in \cF\}$. The Rademacher complexity of the set $l\circ\cF\circ S$ is given by
\begin{equation*}
    R(l\circ\cF\circ S) \hspace{-0.05cm}:=\hspace{-0.05cm} \frac{1}{n}\mathbb{E}_{{\sigma_1,...,\sigma_n}\in\{-1,1\}}\hspace{-0.1cm}\left[\sup_{f\in \mathcal{F}}\sum_{i=1}^n\sigma_il(f(x_i),y_i))\right].
\end{equation*}

For every function $f\in\cF$, the generation error between $\cL_{\cD}(f)$ and $\cL_{S}(f)$ is bounded by the Rademacher complexity of the function space $l\circ \cF\circ S$.
More formally, assume that $\forall f\in\mathcal{F}, \forall x\in\mathcal{X}, |l(f(x),y)|\leq a$. Then with a probability at least $1-\delta$, for all $f\in\mathcal{F}$,
\begin{equation}\label{eqn:generalization error}
    \cL_{\cD}(f)-\cL_S(f)\leq 2\bbE_{S\in\cD^n}[R(l\circ\mathcal{F}\circ S)] + a\sqrt{\frac{2\ln(2/\delta)}{n}}.
\end{equation}
% where $\bbE_{S'\in\cD^n}[R(l\circ\mathcal{F}\circ S')]$ is the 

% \begin{lemma}\label{lm:generalization error} (Theorem 26.5 in the book) Assume $\forall f\in\mathcal{F}, \forall x\in\mathcal{X}, |l(f(x),y)|\leq a$, then with probability at least $1-\delta$, for all $f\in\mathcal{F}$,
% \begin{equation}\label{eqn:generalization error}
%     L_{\cD}(f)-L_S(f)\leq 2\bbE_{S'\in\cD^n}[R(l\circ\mathcal{F}\circ S')] + a\sqrt{\frac{2\ln(2/\delta)}{n}},
% \end{equation}
% % where $\bbE_{S'\in\cD^n}[R(l\circ\mathcal{F}\circ S')]$ is the 
% \end{lemma}

From \autoref{eqn:generalization error}, we can see that given a lower bound of generalization gap $\cL_{\cD}(f)-\cL_S(f)\geq \epsilon$, one has immediately $$ \bbE_{S\in\cD^n}[R(l\circ\mathcal{F}\circ S)]\geq \frac{\epsilon}{2}- \frac{a}{2}\sqrt{\frac{2\ln(2/\delta)}{n}}.$$
Therefore, if we can find the relation between the Rademacher complexity of $l\circ \cF$ and the Lipschitzness of the functions in class $\cF$, we are able to derive a constrain of the Lipschitz constant for $\cF$. The contraction lemma of Rademacher complexity (Lemma 26.9 of \citet{shalev2014understanding}) states that for a given space $A$ and a $L$-lipschitz function $h$ on $A$, we have
$R(h\circ A)\leq L\cdot R(A).$
Thus, if the error function $l(f(x),y)$ is $C$-Lipschitz w.r.t. $f\in\cF$ for arbitrary $y\in[-1,1]$, 
\begin{equation}\label{eqn:contraction}
    R(l\circ\mathcal{F}\circ S)\leq C\cdot R(\mathcal{F}\circ S).
\end{equation}

% The next lemma is a direct result from the contraction lemma (lemma 26.9 of \citet{shalev2014understanding}), which shows the relation between the Rademacher complexity of $l\circ \cF$ and $\cF$, if the error function $l(f(x),y)$ is Lipschitz continuous w.r.t. $f(x)$.
% \begin{lemma}\label{lm:contraction} If $l(f(x),y)$ is $C$-Lipschitz w.r.t. $f(x)$ for arbitrary $y\in[-1,1]$, 
% $$R(l\circ\mathcal{F}\circ S)\leq CR(\mathcal{F}\circ S).$$
% \end{lemma}
% \begin{proof}

% \end{proof}
It has been proved \cite{von2004distance} that the Rademacher complexity of a set is directly related to the number of $\epsilon$-covering of the set. So the first step to calculate the Rademacher complexity of $\cF\circ S$ is to find the covering number of this function space. 

Given a space $(\mathcal{X},||\cdot||)$ and a covering radius $\eta$, let 
$N(\mathcal{X},\eta,||\cdot||)$, a.k.a. the $\eta$-covering number, be the minimum number of $\eta$-ball which covers $\mathcal{X}$. For a given function space $\cF$, define
$$||f-f'||_{\mathcal{F}} = \sup_{x\in\cX}|f(x)-f'(x)|.$$
We have the following upper bound of the covering number of $\cF$:

% \begin{definition}(Covering number) For a given metric space $(\mathcal{X},||\cdot||)$ and a covering radius $\eta$. 
% $N(\mathcal{X},\eta,||\cdot||)$ is the minimum number of $\eta$-ball which cover the whole space of $\mathcal{X}$. Note the metric $||\cdot||_{\mathcal{F}}$ on the functional space $\mathcal{F}$ is $$||f-f'||_{\mathcal{F}} = \sup_{x}|f(x)-f'(x)|.$$

% \end{definition}
\begin{lemma}[Covering number of $L$-Lipschitz function space]
\label{lm: Covering number of functional space}
For a bounded and connected space $(\mathcal{X},||\cdot||)$, let $B_L$ be the set of functions $f$'s such that $\lip_{||\cdot||}(f)\leq L$. If $\mathcal{X}$ is connected and centered, we have for every $\epsilon>0$,
 $$
%  2^{N(\mathcal{X},\frac{\epsilon}{2L},D)}<
 N(B_L,\epsilon,||\cdot||_{\mathcal{F}})\leq\left\lceil\frac{2L\cdot\diam(\mathcal{X})}{\epsilon}\right\rceil2^{N(\mathcal{X},\frac{\epsilon}{2L},||\cdot||)}.$$
 \end{lemma}
The Dudley's integral provides the relation between the covering number of a function class and its Rademacher complexity. With Dudley's integral, \citet{von2004distance} showed that for every $\epsilon>0$,
\begin{equation}\label{eq:Dudley's integral}
\begin{split}
    &\bbE_{S'\in\cD^n}[R(B_L\circ S)]\hspace{-0.05cm}\leq\\
    &\hspace{-0.05cm} 2\epsilon +\frac{4\sqrt{2}}{\sqrt{n}}\hspace{-0.15cm}\int_{\epsilon/4}^{\diam(B_L)}\hspace{-0.3cm}\sqrt{\ln(N(B_L,u,||\cdot||_{\cF}))}du.
\end{split}
\end{equation}
Notice that when $u>2L\cdot\diam(\mathcal{X})$, the number of $u$-covering is $1$ and $\ln(N(B_L,u,||\cdot||_{\cF}))=0$. Combining it with Lemma \ref{lm: Covering number of functional space} yields the following lemma:
\begin{lemma}\label{lm: Rademacher complexity of functional space}Let $(\mathcal{X},||\cdot||)$ be a bounded and connected space and $B_L$ be all functions $f\in\cF$ with $\lip_{||\cdot||}(f)\leq L$. Let $n = |S|$. If $\mathcal{X}$ is connected and centered, for any $\epsilon>0$
  \begin{equation*}
 \begin{split}
 &\bbE_{S\in\cD^n}[R(B_L\circ S)]\leq 2\epsilon + \frac{4\sqrt{2}}{\sqrt{n}} \times\\ &\int_{\epsilon/4}^{2L\cdot\diam(\mathcal{X})}\hspace{-0.35cm}\sqrt{\hspace{-0.05cm}N\hspace{-0.1cm}\left(\mathcal{X},\frac{u}{2L},||\cdot||\hspace{-0.05cm}\right)\hspace{-0.05cm}\ln2\hspace{-0.1cm}+\hspace{-0.1cm}\ln\hspace{-0.1cm}\left\lceil\frac{2L\cdot\diam(\mathcal{X})}{u}\right\rceil\hspace{-0.05cm}}\,du.
\end{split}
  \end{equation*}
 \end{lemma}
 As all the variables in Lemma~\ref{lm: Rademacher complexity of functional space} are known, by calculating the integration, one can derive an upper bound of Rademacher complexity $\bbE_{S\in\cD^n}[R(B_L\circ S)]$:
\begin{lemma}\label{lm:rademacher complexity value} If  $\diam(\mathcal{X})=2$ w.r.t. $||\cdot||$ and $d\geq 3$, we have
  \begin{equation*}
      \begin{split}
          &\bbE_{S\in\cD^n}[R(B_L\circ S)] \leq \\
          &96\frac{L}{n^{1/d}} + \frac{96\sqrt{2\ln2}}{d-2}\frac{L}{n^{1/d}}
    +\frac{16\sqrt{2}L}{\sqrt{n}}\sqrt{\ln\left(\frac{1}{3}n^{1/d}+1\right)}.
    % \\
    % &\sim \cO\left(\frac{L}{n^{1/d}}\right).
%  \cO\left(\frac{L}{n^{1/d}}\right)
      \end{split}
  \end{equation*}
 \end{lemma}
 
 According to Equation 1 in \cite{mendelson2003entropy}, when $\frac{u}{2L}\leq\diam(\cX)$, $N(\cX,\frac{u}{2L},||\cdot||) \leq (\frac{6L\cdot\diam(\cX)}{u})^d$ if  $\cX\subseteq\mathbb{R}^d$. Then the integral part will be
 $\sqrt{\hspace{-0.05cm}(\frac{12L}{u})^d\ln2\hspace{-0.1cm}+\hspace{-0.1cm}\ln\hspace{-0.1cm}\left\lceil\frac{2L\cdot\diam(\mathcal{X})}{u}\right\rceil\hspace{-0.05cm}}$, which is no more than $\sqrt{\hspace{-0.05cm}(\frac{12L}{u})^d\ln2}+\sqrt{\ln(\frac{4L}{u}+1)}$. Taking $\epsilon=\Theta(\frac{L}{n^{1/d}})$, the integral part will be bounded by $\Theta(Ln^{1/2-1/d})$. Thus $\bbE_{S'\in\cD^n}[R(B_L\circ S')]\leq \Theta(\frac{L}{n^{1/d}})+\frac{4\sqrt{2}}{\sqrt{n}}\Theta(Ln^{1/2-1/d}) = \Theta(\frac{L}{n^{1/d}})$.

%  Notice the above corollary only holds with $\lip(l\circ f)\leq L$, but we expect to find the relation between Rademacher complexity and $\lip(f)$. 
In our settings, we are interested in the squared $\ell_2$ loss $l(f(x),y) = (f(x)-y)^2$. We have $\nabla_{f(x)} l(f(x),y)=2(f(x)-y)\leq 2(|f(x)|+|y|)\leq 4$, i.e., $l(f(x),y)$ is 4-Lipschitz w.r.t. $f(x)$ for arbitrary $y\in[-1,1]$. Thus, $\bbE_{S\in\cD^n}[R(l\circ B_L\circ S)]\leq 4\bbE_{S\in\cD^n}[R(B_{L}\circ S)]\hspace{-0.05cm}=\hspace{-0.05cm}\cO\left(\frac{L}{n^{1/d}}\right).$
Combining this result with \autoref{eqn:generalization error} yields the main theorem of our paper:
%  \begin{lemma}\label{lm:liploss}
% If $\forall f\in\mathcal{F}$ is L-lipschitz continuous and $l(f,x,y) = (f(x)-y)^2$, $l\circ \mathcal{F}\subset B_{4L}$. 
% \end{lemma}
 \begin{theorem} [Lipschitzness Lower Bound Beyond the $J$-Lipschitz parametrization
assumption]
 \label{theorem: main result}
Let $\cF$ be any class of functions from $\R^d\rightarrow [-1,1]$ and let $\{(x_i,y_i)\}_{i=1}^n$ be i.i.d. input-output pairs in $\{x:\|x\|\le 1\}\times [-1,1]$ for any given norm $\|\cdot\|$. Assume that:
\begin{itemize}
\vspace{-0.25cm}
\item[1.]
The expected conditional variance of the output (i.e., the ``noise level'') is strictly positive, denoted by $\sigma^2:=\bbE[\var[y|x]]>0$.
\vspace{-0.25cm}
\end{itemize}
Then with probability at least $1-\delta$, for all $f\in\cF$:
\begin{equation*}
\begin{split}
&\frac{1}{n}\sum_{i=1}^n(y_i-f(x_i))^2\le \sigma^2-\epsilon\ \Rightarrow\\ 
&\lip_{\|\cdot\|}(f)\geq\frac{n^{1/d}}{K}\left(\frac{1}{8}\epsilon-\frac{1}{2}\sqrt{\frac{2\ln(2/\delta)}{n}}\right),
% &\lip_{\|\cdot\|}(f)\ge \widetilde\Omega(\epsilon n^{1/d}).
\end{split}
\end{equation*}
where $K =96 + \frac{96\sqrt{2\ln2}}{d-2}
    +\frac{16\sqrt{2}}{n^{1/2-1/d}}\sqrt{\ln(\frac{1}{3}n^{1/d}+1)}.$
% \begin{equation*}
%  \cL_S(f)\leq\sigma^2-\epsilon \Rightarrow \lip_{\|\cdot\|}(f)=\widetilde\Omega(\epsilon n^{1/d}).
% \end{equation*}
\end{theorem}

Theorem \ref{theorem: main result} states that, for all data distribution $\cD$ with label noise of variance $\sigma^2$ and every function $f:\cX\to[-1,1]$, overfitting i.e. $\frac{1}{n}\sum_{i=1}^n(y_i-f(x_i))^2\le \sigma^2-\epsilon$ implies $\lip_{\|\cdot\|}\geq\Omega(\epsilon n^{1/d})$, which validates the second part of our law of robustness, i.e., achieving good robust interpolation is impossible when $n=\exp(\omega(d))$.
\begin{remark}\label{rmk:lawpart1}
    Theorem \ref{theorem: main result} disprove the existence of robust interpolating functions when $n=\exp(\omega(d))$. Thus, the first part of our law of robustness holds only when $n=\poly(d)$.
\end{remark}
%  \begin{theorem}
%  Let $(\mathcal{X},D)$ be a totally bounded metric space and B be a ball of $(Lip(\mathcal{X}),||\cdot||_L)$ with $||f||_{lip}\leq L$, if $\mathcal{X}$ is connected and centered, for any $\epsilon>0$
%  \begin{equation*}
%  \begin{split}
%  &\cR_n(B)\leq 2\epsilon + \frac{4\sqrt{2}}{\sqrt{n}} \times\\ &\int_{\epsilon/4}^{2\diam(\mathcal{X})}\hspace{-0.35cm}\sqrt{\hspace{-0.05cm}N\hspace{-0.1cm}\left(\mathcal{X},\frac{u}{2L},D\hspace{-0.05cm}\right)\hspace{-0.05cm}\ln2\hspace{-0.1cm}+\hspace{-0.1cm}\ln\hspace{-0.1cm}\left(\hspace{-0.05cm}2\hspace{-0.1cm}\left\lceil\frac{2\diam(\mathcal{X})}{u}\right\rceil\hspace{-0.1cm}+\hspace{-0.1cm}1\hspace{-0.05cm}\right)}\,\hspace{-0.1cm}du.
%   \end{split}
%   \end{equation*}
%  \end{theorem}
 
%  \begin{corollary} If $(\mathcal{X},D) = ([0,1]^d,||\cdot||_{\infty})$ and B is a ball of $(Lip(\mathcal{X}),||\cdot||_L)$ with $||f||_{lip}\leq L$, we have $N(\mathcal{X},\frac{u}{2},D)\leq (\frac{2}{u})^d$ and $diam(\mathcal{X})=1$, then
%  $$\cR_n(B)=\cO\left(\frac{L}{n^{1/d}}\right).$$
%  \end{corollary}

\textbf{Tightness of our bound}. Intuitively, the Lipschitzness of the interpolating function is inversely propositional to the distance between the closest training data pairs. Given $n$ training data in the $d$-dimensional bounded space, one can scatter the data evenly in the space, where the distance between any training pair is as large as $\Theta(1/n^{1/d})$.
Inspired by this, we complement Theorem \ref{theorem: main result} with a matching Lipschitzness upper bound of $\cO(n^{1/d})$, which shows that the Lipschitzness lower bound in Theorem \ref{theorem: main result} is achievable by a certain function and training data:

\begin{theorem}[Tightness of our bound]
\label{theorem: upper bound}
For any distribution $\mathcal{D}$ which is supported on $\{x\in\R^d:||x||\leq 1\}$, there exist $n$ training samples $\{x_1,...,x_n\}$ such that $\forall i,j, i\neq j, ||x_i-x_j||\geq\frac{1}{n^{1/d}}$. Denote by $\{y_1,...,y_n\}$ the observed targets. We design a function $f^*$ which first perfectly fits the training samples, i.e., $f^*(x_i)=y_i, \forall i\in[n]$, then use the linear interpolation between neighbour training points as the prediction of other samples.
% i.e. $f(\alpha x_i+(1-\alpha)x_j) = \alpha y_i + (1-\alpha)y_j$.
This function is at most $2n^{1/d}$-Lipschitz.
\end{theorem}
% \begin{proof}
% First, we show that we can find $n$ training samples $\{x_1,...,x_n\}$ such that $\forall i,j, i\neq j, ||x_i-x_j||\geq\frac{1}{n^{1/d}}$. Consider the $\frac{1}{n^{1/d}}$-packing of the space $\{x:||x||\leq 1\}$, the packing number is greater than the $\frac{1}{n^{1/d}}$-covering number of the same space, which at least $(1/\frac{1}{n^{1/d}})^d=n$, we then choose $\{x_1,...,x_n\}$ from the $\frac{1}{n^{1/d}}$-packing, the minimum pairwise distance is at least $\frac{1}{n^{1/d}}$. Next, we show $f^*$ is at most $n^{1/d}$-Lipschitz, as $f^*$ is the linear interpolation between neighbour training points, the worst case Lipschitz constant is $\frac{|y_i-y_j|}{||x_i-x_j||}\leq 2n^{1/d}$.
% \end{proof}
Theorem \ref{theorem: upper bound} shows that there exists $n$ samples, such that the function which perfectly fits the training samples is $\cO(n^{1/d})$-Lipschitz.
% \vspace{-0.2cm}
\subsubsection{Our (counter-intuitive) implications}

It was widely believed that 1) big data~\cite{schmidt2018adversarially}, 2) low dimensionality of input~\cite{blum2020random}, and 3) overparametrization~\cite{bubeck2021universal,pmlr-v134-bubeck21a,gao2019convergence} improve robustness. Our main results of Theorem \ref{theorem: main result} challenge the common beliefs and show that these hypotheses may not be true in the robust interpolation problem. 
Our results shed light on the theoretic understanding of robustness beyond isoperimetry assumption.

\noindent{\textbf{The curse of big data.}} Our Lipschitzness lower bound in Theorem \ref{theorem: main result} is increasing w.r.t. the sample size $n$. The intuition is that as one has more training data, those data are squeezed in the bounded space with smaller margin. Thus to fit the data well, the Lipschitz constant of the interpolating functions cannot be small. Perhaps surprisingly, our results contradict with the common belief that more data always improve model robustness.

\noindent{\textbf{The blessing of dimensionality.}} It is known that high dimensionality of input space strengthens the power of adversary. For example, in the $\ell_\infty$ threat model, an adversary can change every pixel of a given image by 8 or 16 intensity levels. Admittedly, higher dimensionality means that the adversary can modify more pixels. However, we show that our Lipschitzness lower bound in Theorem \ref{theorem: main result} is decreasing w.r.t. $d$. The intuition is that input space with higher dimension has larger space to scatter the data. So the data can be well-separated, and thus the Lipschitz constant of the interpolating functions can be small.

% \noindent{\textbf{Overfitting (due to overparametrization) may hurt robustness.}} Our Lipschitzness lower bound in Theorem \ref{theorem: main result} is increasing w.r.t. the overfitting level $\epsilon$. In fact, one can show that $\frac{1}{n}\sum_{i=1}^n(y_i-f(x_i))^2\le \sigma^2-\epsilon$ implies $p=\widetilde\Omega(n\epsilon^2)$ for all $f\in\cF$ with high probability, if $\cF$ admits a Lipschitz parametrization by $p$ real parameters (see Theorem \ref{theorem: a universal law of data fitting without isoperimetry}). Thus, larger $p$ allows larger $\epsilon$, which leads to larger Lipschitz constant according to Theorem \ref{theorem: main result}. The argument supports a contemporaneous independent work of \citet{hassani2022curse} on the curse of overparametrization in adversarial training and verifies an overfitting concern in the robust deep learning~\cite{rice2020overfitting}.

 % \vspace{-0.2cm}
\section{Small Data May Hurt Performance and Robustness}
\label{section: robustness suffers from too few data}

In Section~\ref{sec:lawofrobust}, we mainly focus on the robust interpolation problem on the training samples. The lower bound given by Theorem \ref{theorem: main result} implies that one can sample at most $\exp(\cO(d))$ training samples in order to obtain an $\cO(1)$-Lipschitz function in the robust interpolation problem. In this section, we show that $n=\exp(\Omega(d))$ is a necessary condition for obtaining a good population error by any $\cO(1)$-Lipschitz learning algorithm.
 
We now provide a complementary result of Section \ref{section: robustness suffers from too many data}. We first prove that for learning algorithms on binary classification tasks, if the number of training samples is less than half of the number of all samples, there exists a distribution with label noise such that the average error of all learning algorithms is greater than a constant. As the distribution on a binary classification is naturally a distribution on the regression tasks, we can find such a distribution for the regression tasks similarly.
\begin{lemma}\label{lemma:freelunch}
Let $\cA(S) : \mathcal{X} \to \{-a,a\}$ be any learning algorithm with respect to the squared $\ell_2$ loss over a domain $\mathcal{X}$ and samples $S$. Assume there are label noise $\mathbb{E}[\var[y|x]]=\sigma^2$. Let $m$ be any number smaller than $|\mathcal{X}|/2$, representing the size of a training set. Then, for any $a>0$ there exists a distribution $\mathcal{D}$ (with label noise) over $\mathcal{X}\times \{-a,a\}$ such that
$$\mathbb{E}_{S\sim \mathcal{D}^m}[\cL_{\mathcal{D}}(\cA(S))]]\geq \frac{1}{2}(a^2+\sigma^2).$$
\end{lemma}
In the next lemma, we will show a no-free-lunch theory on the regression tasks and algorithms that outputs an $L$-Lipschitz function. The intuition is to consider the minimum distance between two points in the distribution $\cD$. On one hand, if the minimum distance is less than $\epsilon$, we can assign the two samples that achieve the minimum distance with labels $1$ and $-1$, respectively. As the algorithm $\cA$ is $L$-Lipschitz, the maximum difference between the predicted labels of the two selected points is $L\epsilon$. Thus, the error of $\cA$ will be larger than $1-L\epsilon$. On the other hand, if the minimum distance is larger than $\epsilon$, the maximum number of points in the distribution $\cD$ will be less than the number of the $\epsilon$-packing of the input space $\cX$. By Lemma \ref{lemma:freelunch}, there exists a distribution such than if the number of training samples is less than half of the $\epsilon$-packing of the input space, the average error of all learning algorithms will be 
at least a constant. More formally, we have the following theorem:

 \begin{lemma}[No-free-lunch theory with $L$-Lipschitz algorithms]\label{lm:freelunchLlip}
Let $\cA(S) : \mathcal{X} \to [-1,1]$ be any algorithm that returns an $L$-Lipschitz function (w.r.t. the norm $\|\cdot\|$) for the task of regression w.r.t. the squared $\ell_2$ loss over a domain $(\mathcal{X},||\cdot||)$ and samples $S$. Let $n$ be the size of training set, i.e., $n=|S|$. Assume that the label noise has variance $\sigma^2 := \mathbb{E}_{\cD}[\var(y|x)]\leq 1/2$.
% any number smaller than $|\mathcal{X}|/2$, representing a training set size.
Then, there exists a distribution $\mathcal{D}$ over $\mathcal{X}\times [-1,1]$ with noisy labels such that for all $L$-Lipschitz (w.r.t. norm $\|\cdot\|$) learning algorithm and any $ \epsilon\in[0,\frac{1}{2L}]$:
% 1. There exists an $O(1)$-Lipschitz function $f : \mathcal{X}\to [-1,1]$ with $L_{\mathcal{D}}(f) = 0$.\\
 \begin{equation*} 
 \begin{split}
 &n<M(\mathcal{X},\epsilon, ||\cdot||)/2\ \Rightarrow\ \\
 &\mathbb{E}_{S\sim \mathcal{D}^n}[\cL_{\mathcal{D}}(\cA(S))]\geq \min\left\{\frac{1}{4},\frac{1}{2}-L\epsilon\right\}+\sigma^2,
   \end{split}
 \end{equation*}
 where $M(\mathcal{X},\epsilon, ||\cdot||)$ is the $\epsilon$-packing number of $(\mathcal{X},||\cdot||)$.
 \end{lemma}
Now we are ready to prove our main theorem.
\begin{theorem}
\label{theorem: sample complexity of robust learning}
Let $S=\{(x_i,y_i)\}_{i=1}^n$ be i.i.d. training pairs in $\{x:\|x\|\le 1\}\times [-1,1]$ for any given norm $\|\cdot\|$. Denote by $\cL_{\cD}(f):=\bbE_{\cD}[(f(x)-y)^2]$ the squared $\ell_2$ loss. Assume that the expected conditional variance of the output (i.e., the ``noise level'') is strictly positive and bounded by $1/2$, denoted by $\sigma^2:=\bbE[\var[y|x]]$. Let $\cA(S) : \mathcal{X} \to \mathbb{R}$ be any $L$-Lipschitz learning algorithm over a training set $S$. Then there exists a distribution $\mathcal{D}'$ of $(x,y)$ such that
\begin{equation*}
    n<\frac{1}{2}\left(\frac{2L}{1-2\epsilon}\right)^d \hspace{-0.1cm}\Rightarrow  \mathbb{E}_{S}[\cL_{\mathcal{D}'}(\cA(S))]\geq \min\left\{\frac{1}{4},\epsilon\right\}+\sigma^2.
\end{equation*}
\end{theorem}
\begin{proof}
     Consider $\mathcal{X} = \{x\in\mathbb{R}^d: ||x||\leq 1\}$. We have $M(\mathcal{X},\eta, ||\cdot||)\geq \left(\frac{1}{\eta}\right)^d$. Thus by Lemma~\ref{lm:freelunchLlip}, there exists a distribution $\mathcal{D}$ such that if $\sigma^2\leq0.5$,
\begin{equation*}
\begin{split}
    &n<\frac{1}{2}\left(\frac{1}{\eta}\right)^d\ \Rightarrow\  n<M(\mathcal{X},\eta, ||\cdot||)/2  \ \Rightarrow\\ 
    &\mathbb{E}_{S\sim \mathcal{D}^n}[\cL_{\mathcal{D}}(\cA(S))]\geq \min\left\{\frac{1}{4},\frac{1}{2}-L\eta\right\}+\sigma^2.
    \end{split}
\end{equation*}
    
Taking $\eta = \frac{1/2-\epsilon}{L}$ where $\epsilon\in(0,1/2)$, we have    
$n<\frac{1}{2}\left(\frac{2L}{1-2\epsilon}\right)^d$ implies $\mathbb{E}_{S\sim \mathcal{D}^n}[\cL_{\mathcal{D}}(\cA(S))]\geq \min\left\{\frac{1}{4},\epsilon\right\}+\sigma^2.$
Thus in the worst case, $n$ has to be at least $\exp(\Omega(d))$ if one wants to achieve good astuteness by any learning algorithm that returns an $\cO(1)$-Lipschitz function. This completes the proof of Theorem \ref{theorem: sample complexity of robust learning}.
\end{proof}
Theorem \ref{theorem: sample complexity of robust learning} states that for certain distributions, $n$ has to be at least $\exp(\Omega(d))$ if one wants to achieve good population error by any $\cO(1)$-Lipschitz learning algorithm. This is not restricted to the algorithms that perfectly fit the training data. The sample complexity lower bound matches the upper bound given in Theorem \ref{theorem: main result}.

\section{Conclusions}
In this work, we study the robust interpolation problem beyond the isoperimetry assumption, and propose a two-fold law of robustness. We show the potential benefit of overparametrization for smooth data interpolation when $n=\poly(d)$, and disprove the potential existence of an $\cO(1)$-Lipschitz robust interpolating function when $n=\exp(\omega(d))$. Besides, we also prove that small data ($\exp(\cO(d))$) may hurt robustness on certain distributions. Perhaps surprisingly, the results shed light on the curse of big data and the blessing of dimensionality regarding robustness.

\section*{Acknowledgement}
Hongyang Zhang is supported by NSERC Discovery Grant RGPIN-2022-03215, DGECR-2022-00357. Yihan Wu and Heng Huang were partially supported by NSF IIS 1838627, 1837956, 1956002, 2211492, CNS 2213701, CCF 2217003, DBI 2225775.

\bibliography{example_paper}

\begin{thebibliography}{35}
\providecommand{\natexlab}[1]{#1}
\providecommand{\url}[1]{\texttt{#1}}
\expandafter\ifx\csname urlstyle\endcsname\relax
  \providecommand{\doi}[1]{doi: #1}\else
  \providecommand{\doi}{doi: \begingroup \urlstyle{rm}\Url}\fi

\bibitem[Azuma(1967)]{azuma1967weighted}
Azuma, K.
\newblock Weighted sums of certain dependent random variables.
\newblock \emph{Tohoku Mathematical Journal, Second Series}, 19\penalty0
  (3):\penalty0 357--367, 1967.

\bibitem[Ben-Tal et~al.(2009)Ben-Tal, El~Ghaoui, and Nemirovski]{ben2009robust}
Ben-Tal, A., El~Ghaoui, L., and Nemirovski, A.
\newblock \emph{Robust optimization}.
\newblock Princeton university press, 2009.

\bibitem[Bhagoji et~al.(2019)Bhagoji, Cullina, and Mittal]{bhagoji2019lower}
Bhagoji, A.~N., Cullina, D., and Mittal, P.
\newblock Lower bounds on adversarial robustness from optimal transport.
\newblock In \emph{Advances in Neural Information Processing Systems}, 2019.

\bibitem[Bhattacharjee et~al.(2021)Bhattacharjee, Jha, and
  Chaudhuri]{bhattacharjee2021sample}
Bhattacharjee, R., Jha, S., and Chaudhuri, K.
\newblock Sample complexity of robust linear classification on separated data.
\newblock In \emph{International Conference on Machine Learning}, pp.\
  884--893, 2021.

\bibitem[Blum et~al.(2020)Blum, Dick, Manoj, and Zhang]{blum2020random}
Blum, A., Dick, T., Manoj, N., and Zhang, H.
\newblock Random smoothing might be unable to certify $\ell_\infty$ robustness
  for high-dimensional images.
\newblock \emph{Journal of Machine Learning Research}, 21:\penalty0 1--21,
  2020.

\bibitem[Bubeck \& Sellke(2023)Bubeck and Sellke]{bubeck2021universal}
Bubeck, S. and Sellke, M.
\newblock A universal law of robustness via isoperimetry.
\newblock \emph{Journal of the ACM}, 70\penalty0 (2):\penalty0 1--18, 2023.

\bibitem[Bubeck et~al.(2021)Bubeck, Li, and Nagaraj]{pmlr-v134-bubeck21a}
Bubeck, S., Li, Y., and Nagaraj, D.~M.
\newblock A law of robustness for two-layers neural networks.
\newblock In \emph{Annual Conference on Learning Theory}, volume 134, pp.\
  804--820, 2021.

\bibitem[Case et~al.(2019)Case, Gallagher, and Gao]{case2019note}
Case, B.~M., Gallagher, C., and Gao, S.
\newblock A note on sub-gaussian random variables.
\newblock \emph{Cryptology ePrint Archive}, 2019.

\bibitem[Cohen et~al.(2019)Cohen, Rosenfeld, and Kolter]{cohen2019certified}
Cohen, J.~M., Rosenfeld, E., and Kolter, J.~Z.
\newblock Certified adversarial robustness via randomized smoothing.
\newblock \emph{ICML}, 2019.

\bibitem[Cullina et~al.(2018)Cullina, Bhagoji, and Mittal]{cullina2018pac}
Cullina, D., Bhagoji, A.~N., and Mittal, P.
\newblock {PAC}-learning in the presence of evasion adversaries.
\newblock In \emph{Advances in Neural Information Processing Systems}, pp.\
  230--241, 2018.

\bibitem[Dan et~al.(2020)Dan, Wei, and Ravikumar]{dan2020sharp}
Dan, C., Wei, Y., and Ravikumar, P.
\newblock Sharp statistical guaratees for adversarially robust gaussian
  classification.
\newblock In \emph{International Conference on Machine Learning}, pp.\
  2345--2355, 2020.

\bibitem[Dobriban et~al.(2020)Dobriban, Hassani, Hong, and
  Robey]{dobriban2020provable}
Dobriban, E., Hassani, H., Hong, D., and Robey, A.
\newblock Provable tradeoffs in adversarially robust classification.
\newblock \emph{arXiv preprint arXiv:2006.05161}, 2020.

\bibitem[Gao et~al.(2019)Gao, Cai, Li, Hsieh, Wang, and
  Lee]{gao2019convergence}
Gao, R., Cai, T., Li, H., Hsieh, C.-J., Wang, L., and Lee, J.~D.
\newblock Convergence of adversarial training in overparametrized neural
  networks.
\newblock \emph{Advances in Neural Information Processing Systems},
  32:\penalty0 13029--13040, 2019.

\bibitem[Goodfellow et~al.(2014)Goodfellow, Shlens, and
  Szegedy]{goodfellow2014explaining}
Goodfellow, I.~J., Shlens, J., and Szegedy, C.
\newblock Explaining and harnessing adversarial examples.
\newblock In \emph{International Conference on Learning Representations}, 2014.

\bibitem[Huber(2004)]{huber2004robust}
Huber, P.~J.
\newblock \emph{Robust statistics}, volume 523.
\newblock John Wiley \& Sons, 2004.

\bibitem[Kumar et~al.(2020)Kumar, Levine, Goldstein, and Feizi]{kumar2020curse}
Kumar, A., Levine, A., Goldstein, T., and Feizi, S.
\newblock Curse of dimensionality on randomized smoothing for certifiable
  robustness.
\newblock In \emph{International Conference on Machine Learning}, pp.\
  5458--5467, 2020.

\bibitem[Li et~al.(2019)Li, Chen, Wang, and Carin]{li2019certified}
Li, B., Chen, C., Wang, W., and Carin, L.
\newblock Certified adversarial robustness with additive noise.
\newblock In \emph{Advances in Neural Information Processing Systems}, pp.\
  9464--9474, 2019.

\bibitem[Madry et~al.(2017)Madry, Makelov, Schmidt, Tsipras, and
  Vladu]{madry2017towards}
Madry, A., Makelov, A., Schmidt, L., Tsipras, D., and Vladu, A.
\newblock Towards deep learning models resistant to adversarial attacks.
\newblock In \emph{International Conference on Learning Representations}, 2017.

\bibitem[Mendelson \& Vershynin(2003)Mendelson and
  Vershynin]{mendelson2003entropy}
Mendelson, S. and Vershynin, R.
\newblock Entropy and the combinatorial dimension.
\newblock \emph{Inventiones mathematicae}, 152\penalty0 (1):\penalty0 37--55,
  2003.

\bibitem[Montasser et~al.(2019)Montasser, Hanneke, and Srebro]{montasser2019vc}
Montasser, O., Hanneke, S., and Srebro, N.
\newblock {VC} classes are adversarially robustly learnable, but only
  improperly.
\newblock In \emph{Annual Conference on Learning Theory}, pp.\  2512--2530,
  2019.

\bibitem[Northcutt et~al.(2021)Northcutt, Athalye, and
  Mueller]{northcutt2021pervasive}
Northcutt, C.~G., Athalye, A., and Mueller, J.
\newblock Pervasive label errors in test sets destabilize machine learning
  benchmarks.
\newblock In \emph{NeurIPS 2021 Datasets and Benchmarks Track}, 2021.

\bibitem[Schmidt et~al.(2018)Schmidt, Santurkar, Tsipras, Talwar, and
  Madry]{schmidt2018adversarially}
Schmidt, L., Santurkar, S., Tsipras, D., Talwar, K., and Madry, A.
\newblock Adversarially robust generalization requires more data.
\newblock In \emph{Advances in Neural Information Processing Systems}, 2018.

\bibitem[Shalev-Shwartz \& Ben-David(2014)Shalev-Shwartz and
  Ben-David]{shalev2014understanding}
Shalev-Shwartz, S. and Ben-David, S.
\newblock \emph{Understanding machine learning: From theory to algorithms}.
\newblock Cambridge university press, 2014.

\bibitem[Szegedy et~al.(2014)Szegedy, Zaremba, Sutskever, Bruna, Erhan,
  Goodfellow, and Fergus]{szegedy2013intriguing}
Szegedy, C., Zaremba, W., Sutskever, I., Bruna, J., Erhan, D., Goodfellow, I.,
  and Fergus, R.
\newblock Intriguing properties of neural networks.
\newblock In \emph{International Conference on Learning Representations}, 2014.

\bibitem[von Luxburg \& Bousquet(2004)von Luxburg and
  Bousquet]{von2004distance}
von Luxburg, U. and Bousquet, O.
\newblock Distance-based classification with {Lipschitz} functions.
\newblock \emph{Journal of Machine Learning Research}, 5:\penalty0 669--695,
  2004.

\bibitem[Wu et~al.(2022{\natexlab{a}})Wu, Huang, Hu, and Huang]{wu2022faster}
Wu, X., Huang, F., Hu, Z., and Huang, H.
\newblock Faster adaptive federated learning.
\newblock \emph{arXiv preprint arXiv:2212.00974}, 2022{\natexlab{a}}.

\bibitem[Wu et~al.(2023)Wu, Hu, and Huang]{wu2023decentralized}
Wu, X., Hu, Z., and Huang, H.
\newblock Decentralized riemannian algorithm for nonconvex minimax problems.
\newblock \emph{arXiv preprint arXiv:2302.03825}, 2023.

\bibitem[Wu et~al.(2021)Wu, Bojchevski, Kuvshinov, and
  G{\"u}nnemann]{wu2021completing}
Wu, Y., Bojchevski, A., Kuvshinov, A., and G{\"u}nnemann, S.
\newblock Completing the picture: Randomized smoothing suffers from the curse
  of dimensionality for a large family of distributions.
\newblock In \emph{International Conference on Artificial Intelligence and
  Statistics}, pp.\  3763--3771. PMLR, 2021.

\bibitem[Wu et~al.(2022{\natexlab{b}})Wu, Bojchevski, and
  Huang]{wu2022adversarial}
Wu, Y., Bojchevski, A., and Huang, H.
\newblock Adversarial weight perturbation improves generalization in graph
  neural network.
\newblock \emph{arXiv preprint arXiv:2212.04983}, 2022{\natexlab{b}}.

\bibitem[Wu et~al.(2022{\natexlab{c}})Wu, Li, Kerschbaum, Huang, and
  Zhang]{wu2022towards}
Wu, Y., Li, X., Kerschbaum, F., Huang, H., and Zhang, H.
\newblock Towards robust dataset learning.
\newblock \emph{arXiv preprint arXiv:2211.10752}, 2022{\natexlab{c}}.

\bibitem[Wu et~al.(2022{\natexlab{d}})Wu, Zhang, and
  Huang]{wu2022retrievalguard}
Wu, Y., Zhang, H., and Huang, H.
\newblock Retrievalguard: Provably robust 1-nearest neighbor image retrieval.
\newblock In \emph{International Conference on Machine Learning}, pp.\
  24266--24279. PMLR, 2022{\natexlab{d}}.

\bibitem[Yang et~al.(2020{\natexlab{a}})Yang, Duan, Hu, Salman, Razenshteyn,
  and Li]{yang2020randomized}
Yang, G., Duan, T., Hu, J.~E., Salman, H., Razenshteyn, I., and Li, J.
\newblock Randomized smoothing of all shapes and sizes.
\newblock In \emph{International Conference on Machine Learning}, pp.\
  10693--10705, 2020{\natexlab{a}}.

\bibitem[Yang et~al.(2020{\natexlab{b}})Yang, Rashtchian, Zhang, Salakhutdinov,
  and Chaudhuri]{yang2020closer}
Yang, Y.-Y., Rashtchian, C., Zhang, H., Salakhutdinov, R., and Chaudhuri, K.
\newblock A closer look at accuracy vs. robustness.
\newblock In \emph{Advances in Neural Information Processing Systems},
  2020{\natexlab{b}}.

\bibitem[Yin et~al.(2019)Yin, Kannan, and Bartlett]{yin2019rademacher}
Yin, D., Kannan, R., and Bartlett, P.
\newblock Rademacher complexity for adversarially robust generalization.
\newblock In \emph{International conference on machine learning}, pp.\
  7085--7094, 2019.

\bibitem[Zhang et~al.(2019)Zhang, Yu, Jiao, Xing, El~Ghaoui, and
  Jordan]{zhang2019theoretically}
Zhang, H., Yu, Y., Jiao, J., Xing, E., El~Ghaoui, L., and Jordan, M.
\newblock Theoretically principled trade-off between robustness and accuracy.
\newblock In \emph{International Conference on Machine Learning}, pp.\
  7472--7482, 2019.

\end{thebibliography}
\bibliographystyle{icml2023}

%%%%%%%%%%%%%%%%%%%%%%%%%%%%%%%%%%%%%%%%%%%%%%%%%%%%%%%%%%%%%%%%%%%%%%%%%%%%%%%
%%%%%%%%%%%%%%%%%%%%%%%%%%%%%%%%%%%%%%%%%%%%%%%%%%%%%%%%%%%%%%%%%%%%%%%%%%%%%%%
% APPENDIX
%%%%%%%%%%%%%%%%%%%%%%%%%%%%%%%%%%%%%%%%%%%%%%%%%%%%%%%%%%%%%%%%%%%%%%%%%%%%%%%
%%%%%%%%%%%%%%%%%%%%%%%%%%%%%%%%%%%%%%%%%%%%%%%%%%%%%%%%%%%%%%%%%%%%%%%%%%%%%%%
\newpage
\appendix
\onecolumn
\section{Missing proofs}
\subsection{Proof of Lemma \ref{lm:concentration}}
\begin{proof}
    
Denote by $X$ the random variable of $\mu$ on bounded space $\cX$ We consider the $Z_1=f(X)$ and $Z_0=\bbE[f(X)]$, since $$|Z_1-Z_0|=|f(X)-\bbE[f(X)]|=|\bbE_{X'}[f(X)-f(X')]|\leq |L\sup_{x,x'\in\cX}||x-x'||| = L\diam(\cX),$$
where $X'$ is of the same distribution with $X$. Because $\bbE[Z_1]=Z_0$, $\{Z_0,Z_1\}$ is a martingale with bounded difference. Thus, by Azuma's inequality Lemma~\ref{lm:azuma}, we have
$$\Pr(|f(x)-\bbE[f(x)]|\geq t) = \Pr(|Z_1-Z_0|\geq t)\leq 2\exp(-\frac{t^2}{2\diam(\cX)^2L^2}).$$
\end{proof}
\subsection{Proof of Lemma \ref{lm:finitefunction}}
\begin{proof}
We use the similar proof technique as in \citet{bubeck2021universal}. Our proof depends on the following lemma.
\begin{lemma}[Lemma 2.1 of \citet{bubeck2021universal}]
\label{lemma: decomposition}
Let $\cF$ be any class of functions from $\R^d\rightarrow [-1,1]$. Let $\{(x_i,y_i)\}_{i=1}^n$ be i.i.d. input-output pairs in $\R^d\times [-1,1]$ for any given norm $\|\cdot\|$. Assume that
the expected conditional variance of the output (i.e., the ``noise level'') is strictly positive, denoted by $\sigma^2:=\bbE[\var[y|x]]>0$.
\begin{equation*}
\Pr\left(\exists f\in\cF:\frac{1}{n}\sum_{i=1}^n (y_i-f(x_i))^2\le\sigma^2-\epsilon\right)\le 2\exp\left(-\frac{n\epsilon^2}{8^3}\right)+\Pr\left(\exists f\in\cF:\frac{1}{n}\sum_{i=1}^n f(x_i)z_i\ge\frac{\epsilon}{4}\right).
\end{equation*}
\end{lemma}

    We now try to bound the term $\Pr(\exists f\in\cF:\frac{1}{n}\sum_{i=1}^n f(x_i)z_i\ge\frac{\epsilon}{4}).$ As $x_i$ is randomly sampled from the input distribution and $\diam(\cX)=2$, we have $$\Pr(|f(x_i)-\bbE[f(x)]|\geq t) \leq 2\exp(-\frac{t^2}{8L^2}),$$
    which indicates $f(x_i)-\bbE[f(x)]$ is $8L^2/n$-subgaussian distributed. Because $|z_i|=|y_i-g(x_i)|\leq2$, we know $(f(x_i)-\bbE[f(x)])z_i$ is $32L^2$-subgaussian. By Property 1 in \citet{case2019note} we know $\frac{1}{n}\sum_{i=1}^n(f(x_i)-\bbE[f(x)])z_i$ is $32L^2/n$-subgaussian. Since $\bbE[(f(x_i)-\bbE[f(x)])z_i]=0$, we have
    $$\Pr\left(\frac{1}{n}\sum_{i=1}^n(f(x_i)-\bbE[f(x)])z_i\geq \frac{\epsilon}{8}\right) \leq \exp(-\frac{n\epsilon^2}{2^{10}L^2}),$$

    Since the range of the functions is in $[-1,1]$ we have $\bbE[f(x)] \in [-1, 1]$ and hence:
    $$\Pr\left(\exists f: \frac{1}{n}\sum_{i=1}^n\bbE[f(x)]z_i\geq \frac{\epsilon}{8}\right) \leq\Pr\left(|\frac{1}{n}\sum_{i=1}^nz_i|\geq \frac{\epsilon}{8}\right),$$
    By Hoeffding’s inequality, the above quantity is smaller than $2 \exp(-n\epsilon^2/8^3)$  Thus we obtain with an union bound:
    \begin{equation*}
        \begin{split}
            \Pr\left(\exists f\in\cF:\frac{1}{n}\sum_{i=1}^n f(x_i)z_i\ge\frac{\epsilon}{4}\right)&\leq |\cF|\Pr\left(\frac{1}{n}\sum_{i=1}^n(f(x_i)-\bbE[f(x)])z_i\geq \frac{\epsilon}{8}\right)+\Pr\left(|\frac{1}{n}\sum_{i=1}^nz_i|\geq \frac{\epsilon}{8}\right)\\
            &\leq |\cF|\exp(-\frac{n\epsilon^2}{2^{10}L^2})+2\exp(-n\epsilon^2/8^3).
        \end{split}
    \end{equation*}
    Together with Lemma~\ref{lemma: decomposition} we have
    \begin{equation*}
    \begin{split}
\Pr\left(\exists f\in\cF:\frac{1}{n}\sum_{i=1}^n (y_i-f(x_i))^2\le\sigma^2-\epsilon\right)
\le 4\exp\left(-\frac{n\epsilon^2}{8^3}\right)+|\cF|\exp\left(-\frac{n\epsilon^2}{2^{10}L^2}\right),
\end{split}
\end{equation*}
which proves this lemma.
\end{proof}

\subsection{Proof of Theorem \ref{thm:result1}}
\begin{proof}We use the similar proof technique as in \citet{bubeck2021universal}.

    We argue that the $\eta$-covering of the function space $\cF$ is upper bounded by the $\eta/J$-covering of the parameter space $\cW$. To see this, we can select the centers $\cW^c=\{w^c_{i}\}$ of the $\eta/J$-covering of $\cW$, and covering $\cF$ with $\eta$-balls centered at $f_{w^c_{i}}$, because $\forall f_w\in\cF$, we can find $w'\in \cW^c$ such that $||w-w'||\leq \eta/J$, by the definition of $J$-Lipschitz parametrization we have $||f_w-f_w'||_{\cF}\leq J||w-w'||\leq\eta$, thus $\cF$ can be covered by $N(\cW,\eta/J,||\cdot||)$ balls. So we have
$$N(\cF,\eta,||\cdot||_{\cF})\leq N(\cW,\eta/J,||\cdot||)\leq (6JW/\eta)^p.$$
Taking $\eta=\frac{\epsilon}{6}$ and denote by $\cW_\epsilon$ the  $\epsilon/6J$-covering of the $\cW$. Applying Lemma~\ref{lm:finitefunction} to $\cF_w = \{f_w:w\in\cW_\epsilon\}$ we have 
\begin{equation*}
    \begin{split}
\Pr\left(\exists f\in\cF_w:\frac{1}{n}\sum_{i=1}^n (y_i-f(x_i))^2\le\sigma^2-\frac{\epsilon}{2}\textrm{ and } \lip_{||\cdot||}(f)\leq L\right)
\le 4\exp\left(-\frac{n\epsilon^2}{8^3}\right)+\exp\left(p\ln(36JW\epsilon^{-1})-\frac{n\epsilon^2}{2^{10}L^2}\right),
\end{split}
\end{equation*}
For all $f\in\cF$, we can find an $f'\in\cF_w$ such that $||f-f_w||_\cF\leq\epsilon/6$. One can easily derive $$\frac{1}{n}\sum_{i=1}^n (y_i-f(x_i))^2\leq \frac{1}{n}\sum_{i=1}^n (y_i-f_w(x_i))^2+\epsilon/2\leq\sigma^2-\epsilon.$$ 
Thus, if $n$ is large enough such that $\exp(-n\epsilon^2/8^3)\leq\delta/8$ and $L\geq\frac{\epsilon}{32}\sqrt{\frac{n}{p\ln(36WJ\epsilon^{-1})+\ln(2/\delta)}}$, we have
$$\Pr\left(\exists f\in\cF:\frac{1}{n}\sum_{i=1}^n (y_i-f(x_i))^2\le\sigma^2-\epsilon\textrm{ and } \lip_{||\cdot||}(f)\leq L\right)\leq\delta,$$
which yields with probability at least $1-\delta$,

$$\frac{1}{n}\sum_{i=1}^n (y_i-f(x_i))^2\le\sigma^2-\epsilon\Rightarrow \lip_{||\cdot||}(f)\geq\frac{\epsilon}{32}\sqrt{\frac{n}{p\ln(36WJ\epsilon^{-1})+\ln(2/\delta)}} $$

\end{proof}

\subsection{Proof of Lemma \ref{lm: Covering number of functional space}}
\begin{proof}
 We consider the Lipschitz function class $B_L:=\{f:\lip_{||\cdot||}(f)\le L\}$. In order to bound the covering number of $\cF$, we consider an $\frac{\epsilon}{2L}$-covering of input space $\cX$ consisting of $N=N_{\epsilon/(2L)}(\cX)$ plates $\cU_1,\cU_2,...,\cU_N$ centered at $s_1,s_2,...,s_N$. The fact that $\cX$ is connected enables one to join any two sets $\cU_i$ and $\cU_j$ by a chain of intersecting $\cU_k$. For any function $f\in\cF$, we can construct its approximating functional $\widetilde f$ by taking its value on $\cU_1$ as an $\epsilon/2$-approximation of $f(s_1)$. As $\diam(\cU_1)\leq L\cdot\diam(\cX)$, there are at most $\lceil 2L\cdot\diam(\cX)/\epsilon\rceil$ such approximations. On the other hand, note that the $N$ plates are chained. By Lipschitzness, the function values of $f$ on $s_1$ and $s_2$ differ at most $\epsilon/2$, and so $f(s_2)$ differs at most $\epsilon$ from $\widetilde f(s_1)$ by triangle inequality. It implies that to construct an $\epsilon$-approximation of $f(s_2)$ on $\cU_2$, we shall know either $\widetilde f(s_1)-\epsilon/2$ or $\widetilde f(s_1)+\epsilon/2$. Repeating the same argument by $N$ times, we can bound the $\epsilon$-covering of $f$ on $\cX$ by $\lceil 2L\cdot\diam(\cX)/\epsilon\rceil2^N$. 
%  Finally, standard covering number generalization bound completes the proof.
 \end{proof}
\subsection{Proof of Lemma \ref{lm: Rademacher complexity of functional space}}
\begin{proof}
The proof of this lemma is quite straight forward. Notice that when $u>2L\cdot\diam(\mathcal{X})$, the number of $u$-covering for $B_L$ is $1$ and $\ln(N(B_L,u,||\cdot||_{\cF}))=0$. Combining \autoref{eq:Dudley's integral} with Lemma \ref{lm: Covering number of functional space} yields this lemma.
\end{proof}

\subsection{Proof of Lemma \ref{lm:rademacher complexity value}}
\begin{proof}
As $\frac{u}{2L}\leq\diam(\cX)$, we have $N(\mathcal{X},\frac{u}{2L},||\cdot||)\leq(\frac{12L}{u})^d$ and
\begin{equation*}
\begin{aligned}
\bbE_{S\in\cD^n}[R(B_L\circ S)]&\leq 2\epsilon + \frac{4\sqrt{2}}{\sqrt{n}}\int_{\epsilon/4}^{2L\cdot\diam(\mathcal{X})}\sqrt{N\left(\mathcal{X},\frac{u}{2L},||\cdot||\right)\ln2+\ln\left(\left\lceil\frac{2L\cdot\diam(\mathcal{X})}{u}\right\rceil\right)}\,du\\
&\leq 2\epsilon + \frac{4\sqrt{2}}{\sqrt{n}}\int_{\epsilon/4}^{4L}\sqrt{\left(\frac{12L}{u}\right)^d\ln2+\ln\left(\left\lceil\frac{2L}{u}\right\rceil\right)}\,du\\
&\leq 2\epsilon + \frac{4\sqrt{2}}{\sqrt{n}}\int_{\epsilon/4}^{4L}\left[\sqrt{\left(\frac{12L}{u}\right)^d\ln2}+\sqrt{\ln\left(\left\lceil\frac{2L}{u}\right\rceil\right)}\right]\,du\\
& \leq 2\epsilon + \frac{4\sqrt{2}}{\sqrt{n}}\int_{\epsilon/4}^{4L}\sqrt{\left(\frac{12L}{u}\right)^d\ln2}\,du+\frac{16\sqrt{2}L}{\sqrt{n}}\sqrt{\ln(16L/\epsilon+1)}.
\end{aligned}
\end{equation*}

Switching the integral variable from $u$ to $v=u/12L$ we have 
\begin{equation*}
\begin{aligned}
\int_{\epsilon/4}^{4L}\sqrt{\left(\frac{12L}{u}\right)^d\ln2}\,du &= 12L \int_{\epsilon/(48L)}^{1/3}\sqrt{v^{-d}\ln2}\,dv\\
    &= 12L\left[\sqrt{\ln2}\frac{1}{-d/2+1}v^{-d/2+1}\lvert_{\epsilon/(48L)}^{1/3})
    % \left[2^{2/d}\sqrt{\ln2}\frac{1}{-d/2+1}u^{-d/2+1}\lvert_{\epsilon/(4L)}^{2/L}+(u\ln(4/u+3)+\frac{4}{3}\ln|3u+4|\lvert)
    \right]\\
    &<12L\frac{2\sqrt{\ln2}}{d-2}
    \left(\frac{48L}{\epsilon}\right)^{d/2-1}.
    \end{aligned}
\end{equation*}
Based on the calculation above we have $$\bbE_{S\in\cD^n}[R(B_L\circ S)] \leq 2\epsilon +L \frac{96\sqrt{2\ln2}}{\sqrt{n}(d-2)}
    \left(\frac{48L}{\epsilon}\right)^{d/2-1}+\frac{16\sqrt{2}L}{\sqrt{n}}\sqrt{\ln(16L/\epsilon+1)}.$$
As this inequality holds for arbitrary $\epsilon>0$, we can take $\epsilon = 48L/n^{1/d}$ and have 
$$\bbE_{S\in\cD^n}[R(B_L\circ S)] \leq 96\frac{L}{n^{1/d}} + \frac{96\sqrt{2\ln2}}{d-2}\frac{L}{n^{1/d}}
    +\frac{16\sqrt{2}L}{\sqrt{n}}\sqrt{\ln\left(\frac{1}{3}n^{1/d}+1\right)}\sim \cO\left(\frac{L}{n^{1/d}}\right).$$
\end{proof}
\subsection{Proof of Theorem \ref{theorem: main result}}
\begin{proof}
According to \autoref{eqn:generalization error}, 
\begin{equation*}
    \cL_{\cD}(f)-\cL_S(f)\leq 2\bbE_{S\in\cD^n}[R(l\circ\mathcal{F}\circ S)] + a\sqrt{\frac{2\ln(2/\delta)}{n}},
\end{equation*}
where $a := \max_{(x,y)}l(f(x),y)\leq 4$.
According to \autoref{eqn:contraction} and $\nabla_{f(x)}l(f(x),y)\leq 4$, we have
% \autoref{cor:rademacher complexity value},
$ \bbE_{S\in\cD^n}[R(l\circ\mathcal{F}\circ S)]\leq 4\bbE_{S\in\cD^n}[R(\mathcal{F}\circ S)].$
Thus,
$$\bbE_{S\in\cD^n}[R(\mathcal{F}\circ S)]\geq \frac{1}{8}\left(\hspace{-0.1cm}\cL_{\cD}(f)\hspace{-0.1cm}-\hspace{-0.1cm}\cL_S(f)\hspace{-0.1cm}-\hspace{-0.1cm}4\sqrt{\frac{2\ln(2/\delta)}{n}}\right).$$
Under the label noise settings, we have
\begin{equation*}
\begin{split}
\cL_{\cD}(f)&=\bbE_{\cD}[(f(x)-y)^2]\\
&=\bbE_{x,y}[(f(x)-\bbE_y[y|x])^2+(y-\bbE_y[y|x])^2]\\
&\geq \bbE_{x}[\var(y|x)]=\sigma^2.
\end{split}
\end{equation*}
So with the overfitting assumption $\cL_S(f)\leq\sigma^2-\epsilon$, we have
\begin{equation}\label{eqn:rademacher bound}
\begin{split}
    \bbE_{S\in\cD^n}[R(\mathcal{F}\circ S)]&\geq \frac{1}{8}\left(\hspace{-0.1cm}\cL_{\cD}(f)\hspace{-0.1cm}-\hspace{-0.1cm}\cL_S(f)\hspace{-0.1cm}-\hspace{-0.1cm}4\sqrt{\frac{2\ln(2/\delta)}{n}}\right)\\
    &=\frac{\epsilon}{8} - \frac{1}{2}\sqrt{\frac{2\ln(2/\delta)}{n}}.
\end{split}
\end{equation}
Consider $B_L=\{f\in\mathcal{F}: \lip_{||\cdot||}(f)\leq L\}$.
% If $\lip_{||\cdot||}(f)= L$ for $f\in B_L$.
According to Lemma \ref{lm:rademacher complexity value}, we have
 \begin{equation*}
 \begin{split}
K\frac{L}{n^{1/d}}\geq\bbE_{S\in\cD^n}[R(B_L\circ S)]\geq\frac{\epsilon}{8} - \frac{1}{2}\sqrt{\frac{2\ln(2/\delta)}{n}},
      \end{split}
 \end{equation*}
 where $K =96 + \frac{96\sqrt{2\ln2}}{d-2}
    +\frac{16\sqrt{2}}{n^{1/2-1/d}}\sqrt{\ln(\frac{1}{3}n^{1/d}+1)}\sim\Theta(1)$.
%  which yields
%  \begin{equation*}
%      R(B_{4L}\circ S)\geq\frac{\epsilon}{2}-16\sqrt{\frac{2\ln(4/\delta)}{n}},
%  \end{equation*}
Thus we have 
% $$\lip_{||\cdot||}(f) \hspace{-0.05cm}=\hspace{-0.05cm}L\geq\hspace{-0.05cm} \frac{n^{1/d}}{K}\hspace{-0.05cm}\left(\frac{1}{8}\epsilon\hspace{-0.05cm}-\hspace{-0.05cm}\frac{1}{2}\sqrt{\frac{2\ln(2/\delta)}{n}}\right)\hspace{-0.05cm}\sim \widetilde\Omega(\epsilon n^{1/d}).$$
$$L\geq\hspace{-0.05cm} \frac{n^{1/d}}{K}\hspace{-0.05cm}\left(\frac{1}{8}\epsilon\hspace{-0.05cm}-\hspace{-0.05cm}\frac{1}{2}\sqrt{\frac{2\ln(2/\delta)}{n}}\right).$$
If $\exists f_0\in\cF$, such that 
\begin{equation*}
 \cL_S(f_0)\hspace{-0.1cm}\leq\hspace{-0.1cm}\sigma^2-\epsilon \Rightarrow \lip_{\|\cdot\|}(f_0)\hspace{-0.1cm}<\hspace{-0.1cm}\frac{n^{1/d}}{K}\hspace{-0.1cm}\left(\frac{1}{8}\epsilon\hspace{-0.1cm}-\hspace{-0.1cm}\frac{1}{2}\sqrt{\frac{2\ln(2/\delta)}{n}}\right),
\end{equation*}
we have
 \begin{equation*}
 \begin{split}
&\frac{\epsilon}{8} - \frac{1}{2}\sqrt{\frac{2\ln(2/\delta)}{n}}>K\frac{\lip_{\|\cdot\|}(f_0)}{n^{1/d}}\geq\\
&\bbE_{S\in\cD^n}[R(B_{\lip_{\|\cdot\|}(f_0)}\circ S)]\geq\frac{\epsilon}{8} - \frac{1}{2}\sqrt{\frac{2\ln(2/\delta)}{n}},
      \end{split}
 \end{equation*}
which yields contradiction. Therefore, $\forall f\in\cF$,
% \begin{equation*}
%  L_S(f)\leq\sigma^2-\epsilon \Rightarrow \lip_{\|\cdot\|}(f)\geq\widetilde\Omega((\epsilon-\sqrt\frac{\ln(1/\delta)}{n}) n^{1/d})
% \end{equation*}
\begin{equation*}
 \cL_S(f)\hspace{-0.1cm}\leq\hspace{-0.1cm}\sigma^2-\epsilon \Rightarrow \lip_{\|\cdot\|}(f)\hspace{-0.1cm}\geq\hspace{-0.1cm}\frac{n^{1/d}}{K}\hspace{-0.1cm}\left(\frac{1}{8}\epsilon\hspace{-0.1cm}-\hspace{-0.1cm}\frac{1}{2}\sqrt{\frac{2\ln(2/\delta)}{n}}\right).
\end{equation*}
Taking $\mathcal{X} = \{ x\in\mathbb{R}^d: ||x||\leq1\}$, we have $\diam(\cX)=2$, which yields Theorem \ref{theorem: main result}.
\end{proof}

\subsection{Proof of Theorem \ref{theorem: upper bound}}
\begin{proof}
First, we show that we can find $n$ training samples $\{x_1,...,x_n\}$ such that $\forall i,j, i\neq j, ||x_i-x_j||\geq\frac{1}{n^{1/d}}$. Consider the $\frac{1}{n^{1/d}}$-packing of the space $\{x:||x||\leq 1\}$, the packing number is greater than the $\frac{1}{n^{1/d}}$-covering number of the same space, which at least $(1/\frac{1}{n^{1/d}})^d=n$, we then choose $\{x_1,...,x_n\}$ from the $\frac{1}{n^{1/d}}$-packing, the minimum pairwise distance is at least $\frac{1}{n^{1/d}}$. Next, we show $f^*$ is at most $n^{1/d}$-Lipschitz, as $f^*$ is the linear interpolation between neighbour training points, the worst case Lipschitz constant is $\frac{|y_i-y_j|}{||x_i-x_j||}\leq 2n^{1/d}$.
\end{proof}
\subsection{Proof of Lemma \ref{lemma:freelunch}}
\begin{proof}

% \begin{equation*}
% \begin{aligned}
% \mathbb{E}_{S\sim \mathcal{D}^m}[\mathbb{E}_Z[\cL_{\mathcal{D}+(0,Z)}(\cA(S))]]&=\mathbb{E}_{S\sim \mathcal{D}^m}[\mathbb{E}_{z\sim Z,(x,y)\sim\mathcal{D}}[(\cA(S)(x)-y+z)^2]]\\
% &=\mathbb{E}_{S\sim \mathcal{D}^m}[\mathbb{E}_{z\sim Z,(x,y)\sim\mathcal{D}}[(\cA(S)(x)-y+z)^2]]\\
% & = \mathbb{E}_{S\sim \mathcal{D}^m}[\mathbb{E}_{z\sim Z,(x,y)\sim\mathcal{D}}[(\cA(S)(x)-y)^2]]+\mathbb{E}_{S\sim \mathcal{D}^m}[\mathbb{E}_{z\sim Z}[z^2]]\\&+2\mathbb{E}_{z\sim Z}[z]\mathbb{E}_{S\sim \mathcal{D}^m}[\mathbb{E}_{(x,y)\sim\mathcal{D}}[\cA(S)(x)-y]]\\
% &=\mathbb{E}_{S\sim \mathcal{D}^m}[\cL_{\mathcal{D}}(\cA(S))]+\sigma^2
% \end{aligned}
% \end{equation*}

Our proof is partly based on Theorem 5.1 of \citet{shalev2014understanding}. Let $\cC$ be a subset of $\mathcal{X}$ of size $2m$. 
% The intuition of the proof is that any learning algorithm that observes only half of the instances in $C$ has no information on what should be the labels of the rest of the instances in $C$. Therefore, there exists a “reality,” that is, some target function $f$, that would contradict the labels that $A(S)$ predicts on the unobserved instances in $C$.
There exist $T=2^{2m}$ possible labeling functions from $\mathcal{C}$ to $\{-a,a\}$. Denote these functions by $f_1,...,f_T$. We then define a distribution $\mathcal{D}_i$ w.r.t. $f_i$ by
\begin{equation*}
\mathcal{D}_i(\{(x,y)\})=
\begin{cases}
    p/|\cC|, & \text{if } y=f_i(x);\\
    (1-p)/|\cC|, & \text{if } y\neq f_i(x),
\end{cases}
\end{equation*}
where $p>1/2$ satisfies $\var(y|x)=\sigma^2=4a^2p(1-p)$ (notice that as $f_i(x)$ can only be $a$ or $-a$, $p$ is the same for all $f_i(x)$'s). In this way, $\cD_{i}$ satisfies the noisy label setting.
We will show that for every algorithm $\cA$ that receives a training set of size $m$ from $\mathcal{C}\times \{-a,a\}$ and returns a function $\cA(S) : \cC \to \mathbb{R}$ , it holds that 
\begin{equation*}
\max_{i\in[T]}\mathbb{E}_{S\sim \mathcal{D}_i^m}[\cL_{\mathcal{D}_i}(\cA(S))]\geq \frac{a^2+\sigma^2}{2}.
\end{equation*}
There are $k = (2m)^m$ possible sequences of $m$ instances from $\cC$. Denote these sequences by $S_1, . . . , S_k$. Also, if $S_j =(x_1,...,x_m)$, we denote by $S_j^i$ the sequence containing the instances in $S_j$ labeled by the function $f_i$, namely, $S_j^i = ((x_1, a_1f_i(x_1)), . . . , (x_m, a_mf_i(x_m)))$, where $\Pr(a_l=1)=p$, $\Pr(a_l=-1)=1-p$, and $a_1,...,a_m $ are i.i.d. for all $S_j^i$, given that $p$ is the same for all $f_i(x)$'s. If the distribution is $\mathcal{D}_i$, then the possible training sets that algorithm $\cA$ receives are $S_1^i , . . . , S_k^i$, and all these training sets have the same probability of being sampled. Therefore,

$$\mathbb{E}_{S\sim \mathcal{D}_i^m}[\cL_{\mathcal{D}_i}(\cA(S))]=\frac{1}{k}\sum_{j=1}^{k}\cL_{\mathcal{D}_i}(\cA(S_j^i)).$$
% Besides, as each $S_j^i$ is depend on $\bm{a}$, $\cL_{\mathcal{D}_i}(\cA(S_j^i)) = \bbE_{\bm{a}}[(\cA(S_j^i(\bm{a}))(x_i)-a_)]$.
Using the facts that ``maximum'' is larger than ``average'' and that ``average'' is larger than ``minimum'', we have

\begin{equation*}
    \begin{aligned}
    \max_{i\in[T]} \frac{1}{k}\sum_{j=1}^{k}\cL_{\mathcal{D}_i}(\cA(S_j^i))&\geq \frac{1}{T}\sum_{i=1}^T\frac{1}{k}\sum_{j=1}^{k}\cL_{\mathcal{D}_i}(\cA(S_j^i))\\
    &=\frac{1}{k}\sum_{j=1}^{k}\frac{1}{T}\sum_{i=1}^T\cL_{\mathcal{D}_i}(\cA(S_j^i))\\
    &\geq \min_{j\in[k]}\frac{1}{T}\sum_{i=1}^T\cL_{\mathcal{D}_i}(\cA(S_j^i)).
    % &= \min_{j\in[k]}\frac{1}{T}\sum_{i=1}^T\bbE_{\bm{a}}[\cA(S_j^i(\bm{a}))]\\
    % &= \min_{j\in[k]}\bbE_{\bm{a}}[\frac{1}{T}\sum_{i=1}^T\cA(S_j^i(\bm{a}))]\\
    % &\geq \min_{j\in[k],\bm{a}\in\{-1,1\}^m}\frac{1}{T}\sum_{i=1}^T\cA(S_j^i(\bm{a}))\\
    \end{aligned}
\end{equation*}
Next, fix some $j \in [k]$ 
% and $\bm{a}\in \{-1,1\}^m$
. Denote by $S_j := (x_1,...,x_m)$ and let $v_1,...,v_q$ be the instances in $\cC$ that do not appear in $S_j$. Clearly, $q\geq m$. Therefore, for every function $h : \cC \to \mathbb{R}$ and every $i$ we have
\begin{equation*}
\begin{split}
    \cL_{\mathcal{D}_i}(h) &= \frac{1}{2m}\bbE_{\bm{a}\in\{-1,1\}^{2m}}\left[\sum_{x\in \cC}(h(x)-a_if_i(x))^2\right]\\
    &= \frac{1}{2m}\sum_{x\in \cC}[p(h(x)-f_i(x))^2+(1-p)(h(x)+f_i(x))^2]\\
    &= \frac{1}{2m}\sum_{x\in \cC}[(h(x)-(2p-1)f_i(x))^2+4p(1-p)f_i(x)^2]\\
    &= \sigma^2+\frac{1}{2m}\sum_{x\in \cC}[(h(x)-(2p-1)f_i(x))^2].
    \end{split}
\end{equation*}

Note that
\begin{equation*}
    \frac{1}{2m}\sum_{x\in \cC}[(h(x)-(2p-1)f_i(x))^2]\geq\frac{1}{2m}\sum_{r=1}^q(h(v_r)-(2p-1)f_i(v_r))^2\geq \frac{1}{2q}\sum_{r=1}^q(h(v_r)-(2p-1)f_i(v_r))^2.
\end{equation*}
Hence,
\begin{equation*}
\begin{split}
\frac{1}{T}\sum_{i=1}^T\cL_{\mathcal{D}_i}(\cA(S_j^i))&\geq \frac{1}{T}\sum_{i=1}^T\bbE_{\bm{a}\in\{-1,1\}^{m}}\left[\sigma^2+\frac{1}{2q}\sum_{r=1}^q(\cA(S_j^i(\bm{a}))(v_r)-(2p-1)f_i(v_r))^2\right]\\
&=\sigma^2+  \frac{1}{2q}\sum_{r=1}^q\frac{1}{T}\sum_{i=1}^T\bbE_{\bm{a}\in\{-1,1\}^{m}}[(\cA(S_j^i(\bm{a}))(v_r)-(2p-1)f_i(v_r))^2]\\
&\geq\sigma^2+ \frac{1}{2} \min_{r\in[p]} \frac{1}{T}\sum_{i=1}^T\bbE_{\bm{a}\in\{-1,1\}^{m}}[(\cA(S_j^i)(\bm{a})(v_r)-(2p-1)f_i(v_r))^2].
\end{split}
\end{equation*}
Next, fix some $r \in [p]$. We can partition all the functions in $f_1,...,f_T$ into $T/2$ disjoint pairs, where for a pair $(f_i, f_{i'})$ we have that for every $c \in \cC, f_i(c) \neq f_{i'}(c)$ if and only if $c=v_r$. Note that for such a pair and the same $\bm{a}$, we must have $S_{j}^i(\bm{a}) =S_j^{i'}(\bm{a})$ and $\forall \bm{a}\in\{-1,1\}^{m},\Pr(\bm{a}|S_{j}^i)=\Pr(\bm{a}|S_{j}^{i'})$. It follows that
\begin{equation*}
\begin{split}&\bbE_{\bm{a}\in\{-1,1\}^{m}}[(\cA(S_j^i)(v_r)-(2p-1)f_i(v_r))^2]+ \bbE_{\bm{a}\in\{-1,1\}^{m}}[(\cA(S_j^{i'})(v_r)-(2p-1)f_{i'}(v_r))^2]\\
\geq &\bbE_{\bm{a}\in\{-1,1\}^{m}}[(\cA(S_j^i)(v_r)-(2p-1)f_i(v_r))^2+(\cA(S_j^{i'})(v_r)-(2p-1)f_{i'}(v_r))^2]\\
\geq&\bbE_{\bm{a}\in\{-1,1\}^{m}}\left[\frac{1}{2}(2p-1)^2(f_{i'}(v_r)-f_i(v_r))^2\right]\\
= &2(2p-1)^2a^2,
\end{split}
\end{equation*}
which yields
\begin{equation*}\frac{1}{T}\sum_{i=1}^T\bbE_{\bm{a}\in\{-1,1\}^{m}}[(\cA(S_j^i(\bm{a}))(v_r)-(2p-1)f_i(v_r))^2]\geq 
% \frac{a^2}{2}
(2p-1)^2a^2.
\end{equation*}
Combining the discussion above, we have
\begin{equation*} \max_{i\in[T]}\mathbb{E}_{S\sim \mathcal{D}_i^m}[\cL_{\mathcal{D}_i}(\cA(S))]\geq\min_{j\in[k]}\frac{1}{T}\sum_{i=1}^T\cL_{\mathcal{D}_i}(\cA(S_j^i))\geq \sigma^2+\frac{1}{2}(2p-1)^2a^2=\frac{a^2+\sigma^2}{2}.
\end{equation*}
\end{proof}
\subsection{Proof of Lemma \ref{lm:freelunchLlip}}
\begin{proof} Consider an arbitrary finite set $\cC\subseteq\mathcal{X}$. Denote by $d(\cC):=\min_{(a,b)\in \cC\times \cC, a\neq b}||a-b||$. We now consider two cases: a) $d(\cC)< \epsilon$ and b) $d(\cC)\geq \epsilon$, and show that our conclusion holds for both cases.

% $$\mathbb{E}_Z\mathbb{E}_{S\sim \mathcal{D}^n}[\cL_{\mathcal{D}+(0,Z)}(\cA(S))] =  \mathbb{E}_{S\sim \mathcal{D}^n}[\cL_{\mathcal{D}}(\cA(S))]+\sigma^2$$

Case a): $d(\cC)< \epsilon$. Denote by $(x_1,x_2) = \argmin_{(a,b)\in \cC\times \cC, a\neq b}||a-b||$. We can select $\mathcal{D}$ such that $\mathcal{D}(\{(x_1,1)\}) = \frac{p}{2},\mathcal{D}(\{(x_1,-1)\}) = \frac{(1-p)}{2}$ and $\mathcal{D}(\{(x_2,-1)\}) = \frac{p}{2},\mathcal{D}(\{(x_2,-1)\}) = \frac{1-p}{2}$, where $4p(1-p)=\sigma^2,p> 1/2$. Consider an $L$-Lipschitz learning algorithm $\cA(S) : \cC \to \mathbb{R}$:
% , as the support of $C$ in $\mathcal{D}$ only contains two points, the Lipschitz constant of $C$ is $L = |A(S)(x_1)-A(S)(x_2)|/||x_1-x_2||$, then 
\begin{equation*} 
\begin{aligned}
&\quad\mathbb{E}_{S\sim \mathcal{D}^n}[\cL_{\mathcal{D}}(\cA(S))]\\ &\geq \min_{S\sim \mathcal{D}^n} \left[\frac{p}{2}(\cA(S)(x_1)-1)^2+\frac{1-p}{2}(\cA(S)(x_1)+1)^2+\frac{p}{2}(\cA(S)(x_2)+1)^2+\frac{1-p}{2}(\cA(S)(x_2)-1)^2\right]\\
&\geq
% \sigma^2+ (\cA(S)(x_1)-(2p-1))^2+(\cA(S)(x_2)+(2p-1))^2
\min_{S\sim \mathcal{D}^n} [1-(2p-1)|\cA(S)(x_1)-\cA(S)(x_2)|]\\
&\geq 1-L(2p-1)||x_1-x_2||\\
&\geq 1-L\cdot d(\cC)\\
&=1-L\epsilon\\
&\geq \frac{1}{2}-L\epsilon+\sigma^2.
% &> 1-L(2p-1)\epsilon=1-L\sqrt{1-\sigma^2}\epsilon.
\end{aligned}
\end{equation*}

Case b): $d(\cC)\geq \epsilon$. We reduce the regression problem from a binary classification problem with target $\{-1,1\}$ by considering the distribution $\mathcal{D}$ such that $\cD$ only on $\mathcal{X}\times \{-1,1\}$. Then by \autoref{lm:freelunch},  for every $\cA(S) : \mathcal{X} \to \mathbb{R}$ and every $\cC\subseteq \cX$ there exists $\mathcal{D}$ such that \begin{equation*} 
n<\frac{|\cC|}{2}\ \Rightarrow\  \mathbb{E}_{S\sim \mathcal{D}^n}[\cL_{\mathcal{D}}(\cA(S))] \geq \frac{1+\sigma^2}{2}.
\end{equation*}
Notice that $\cC \subseteq \mathcal{X}$ can be chosen arbitrarily. Thus we have 
\begin{equation*}
n<\max_{\cC \subseteq \mathcal{X},d(\cC)\geq \epsilon}\frac{|\cC|}{2}\ \Rightarrow\  \mathbb{E}_{S\sim \mathcal{D}^n}[\cL_{\mathcal{D}}(\cA(S))] \geq \frac{1+\sigma^2}{2}.
\end{equation*}
Denote the $\epsilon$-packing number of space $(\mathcal{X},||\cdot||)$ by $M(\mathcal{X},\epsilon, ||\cdot||)$. We have $$\max_{\cC \subseteq \mathcal{X},d(\cC)\geq \epsilon}\frac{|\cC|}{2}=M(\mathcal{X},\epsilon, ||\cdot||)/2.$$
That is,
\begin{equation*} 
n< M(\mathcal{X},\epsilon, ||\cdot||)/2\ \Rightarrow\  \mathbb{E}_{S\sim \mathcal{D}^n}[\cL_{\mathcal{D}}(\cA(S))] \geq \frac{1+\sigma^2}{2}\geq \frac{1}{4}+\sigma^2.
\end{equation*}
Combining a) and b) yields our conclusion.
\end{proof}
\subsection{Proof of Theorem \ref{theorem: sample complexity of robust learning}}

\begin{proof}
Consider $\mathcal{X} = \{x\in\mathbb{R}^d: ||x||\leq 1\}$. We have $M(\mathcal{X},\eta, ||\cdot||)\geq \left(\frac{1}{\eta}\right)^d$ and thus there exists a distribution $\mathcal{D}$ such that if $\sigma^2\leq0.5$
\begin{equation*}
\begin{split}
    n<\frac{1}{2}\left(\frac{1}{\eta}\right)^d\ \Rightarrow\  n<M(\mathcal{X},\eta, ||\cdot||)/2  \ \Rightarrow\ \mathbb{E}_{S\sim \mathcal{D}^n}[\cL_{\mathcal{D}}(\cA(S))]\geq \min\left\{\frac{1}{4},\frac{1}{2}-L\eta\right\}+\sigma^2.
    \end{split}
\end{equation*}
    
Taking $\eta = \frac{1/2-\epsilon}{L}$ where $\epsilon\in(0,1/2)$, we have    
\begin{equation*}
\begin{split}
    n<\frac{1}{2}\left(\frac{2L}{1-2\epsilon}\right)^d\ \Rightarrow\ \mathbb{E}_{S\sim \mathcal{D}^n}[\cL_{\mathcal{D}}(\cA(S))]\geq \min\left\{\frac{1}{4},\epsilon\right\}+\sigma^2.
    \end{split}
    \end{equation*}
Thus in the worst case, $n$ has to be at least $\exp(\Omega(d))$ if one wants to achieve good astuteness by any $\cO(1)$-Lipschitz learning algorithm, this completes our proof.
\end{proof}

\section{Some basic concepts of Rademacher complexity}
\begin{definition}[Representativeness of $S$]
$$Rep_{\cD}(l,\mathcal{F},S):= \sup_{f\in \mathcal{F}}(\cL_D(f)-\cL_S(f)).$$
\end{definition}

\begin{definition}[Rademacher complexity] For $A\in\mathbb{R}^n$,
$$R(A):= \frac{1}{n}\mathbb{E}_{\sigma_1,...,\sigma_n\in\{-1,1\}}\left[\sup_{f\in \mathcal{F}}\sum_{i=1}^n\sigma_ia_i\right].$$
\end{definition}

\begin{lemma}  Assume that $\forall f\in\mathcal{F}, \forall x\in\mathcal{X}, |l(f,x)|\leq c$. Then with probability at least $1-\delta$, for all $f\in\mathcal{F}$,
$$\cL_\cD(f)-\cL_S(f)\leq \mathbb{E}_{S\in \cD^n}[Rep_{D}(l,\mathcal{F},S)] + c\sqrt{\frac{2\ln(2/\delta)}{n}}.$$
\end{lemma}

\begin{lemma}[Lemma 26.2 in \citet{shalev2014understanding}]
$$\mathbb{E}_{S\in D^n}[Rep_{\cD}(l,\mathcal{F},S)]\leq2\mathbb{E}_{S\in \cD^n}[R(l\circ\mathcal{F}\circ S)],$$
where $S=\{x_1,...,x_n\}$ and $l\circ\mathcal{F}\circ S=\{(l(f,x_1,y_1),...,l(f,x_n,y_n))\in\mathbb{R}^n\}$.
\end{lemma}
\begin{lemma}[Theorem 26.5 in \citet{shalev2014understanding}]\label{lm:generalization error} Assume $\forall f\in\mathcal{F}, \forall x\in\mathcal{X}, |l(f,x)|\leq a$, then with probability at least $1-\delta$, for all $f\in\mathcal{F}$,
$$\cL_\cD(f)-\cL_S(f)\leq 2\mathbb{E}_{S'\in \cD^n}[R(l\circ\mathcal{F}\circ S')] + a\sqrt{\frac{2\ln(2/\delta)}{n}}.$$
\end{lemma}
\begin{lemma}[Lemma 26.9 in \citet{shalev2014understanding}]\label{lm:contraction} If $l(f(x),y)$ is $C_{||\cdot||}$-Lipschitz w.r.t. $f(x)$ for arbitrary $y\in[-1,1]$, 
$$R(l\circ\mathcal{F}\circ S)\leq C\cdot R(\mathcal{F}\circ S).$$
\end{lemma}

\end{document}